\documentclass[11pt]{article}

\usepackage{amsmath,amsfonts,bm}

\def\eqref#1{equation~\ref{#1}}

\def\1{\bm{1}}

\DeclareMathAlphabet{\mathsfit}{\encodingdefault}{\sfdefault}{m}{sl}
\SetMathAlphabet{\mathsfit}{bold}{\encodingdefault}{\sfdefault}{bx}{n}

\newcommand{\E}{\mathbb{E}}

\newcommand{\R}{\mathbb{R}}

\usepackage{fullpage}
\PassOptionsToPackage{authoryear,numbers,sort&compress}{natbib}
\usepackage{natbib}

\usepackage[modulo]{lineno}

\usepackage[utf8]{inputenc}
\usepackage[T1]{fontenc}

\usepackage{algorithm}
\usepackage{algorithmic}

\usepackage{amssymb}
\usepackage{mathtools}
\usepackage{amsthm}
\usepackage{tikz}
\usepackage{nicefrac}
\usepackage{thm-restate}
\usepackage{wrapfig}
\usepackage{amartya_ltx} 
\usepackage{multirow}
\usepackage{enumitem}

\usepackage{hyperref}
\usepackage{url}
\usepackage{graphicx}
\usepackage{caption}     
\usepackage{subcaption}  
\usepackage[nameinlink,noabbrev]{cleveref} 
\usepackage{booktabs}
\usepackage{wrapfig}

\crefname{thm}{theorem}{theorems}
\Crefname{thm}{Theorem}{Theorems}
\crefname{defn}{definition}{definitions}
\Crefname{defn}{Definition}{Definitions}
\crefname{lem}{lemma}{lemmas}
\Crefname{lem}{Lemma}{Lemmas}
\crefname{corollary}{corollary}{corollaries}
\Crefname{corollary}{Corollary}{Corollaries}
\crefname{claim}{claim}{claims}
\Crefname{claim}{Claim}{Claims}
\crefname{proposition}{proposition}{propositions}
\Crefname{proposition}{Proposition}{Propositions}
\crefname{remark}{remark}{remarks}
\Crefname{remark}{Remark}{Remarks}
\crefname{example}{example}{examples}
\Crefname{example}{Example}{Examples}
\crefname{question}{question}{questions}
\Crefname{question}{Question}{Questions}
\crefname{inequality}{inequality}{inequalities}
\Crefname{inequality}{Inequality}{Inequalities}
\crefname{assumption}{assumption}{assumptions}
\Crefname{assumption}{Assumption}{Assumptions}

\usepackage{cleveref}

\usepackage{authblk}

\usepackage{amartya_ltx}

\let\epsilon\varepsilon
\newcommand{\rhoalign}[0]{\rho}

\usepackage{xparse}

\let\OrigVspace\vspace
\NewDocumentCommand{\NoNegVspace}{s m}{%
  \begingroup
  \dimen0=#2\relax
  \ifdim\dimen0<0pt
  \else
    \IfBooleanTF{#1}{\OrigVspace*{#2}}{\OrigVspace{#2}}%
  \fi
  \endgroup
}

\newcommand{\inputNoNegVspace}[1]{%
  \begingroup
  \let\vspace\NoNegVspace
  \input{#1}%
  \endgroup
}

\newcommand{\inputNoWrapNoNegVspace}[1]{%
  \begingroup
  \let\vspace\NoNegVspace

  \RenewDocumentEnvironment{wrapfigure}{O{} m m}{%
    \begin{figure}[t]
      \hfill
      \begin{minipage}{##3}
        \centering
  }{%
      \end{minipage}
    \end{figure}
  }%

  \input{#1}%
  \endgroup
}

\title{LoRA and Privacy: \\ When Random Projections Help (and When They Don’t)}

\author[1]{Yaxi Hu$^*$}
\author[2]{Johanna D\"ungler\thanks{\{yaxi.hu, bernhard.schoelkopf\}@tuebingen.mpg.de, \{jodu, amsa\}@di.ku.dk}}
\author[1]{Bernhard Sch\"olkopf}
\author[2]{Amartya Sanyal}
\affil[1]{Max Planck Institute for Intelligent Systems, Tübingen, Germany}
\affil[2]{Department of Computer Science, University of Copenhagen}
\affil[*]{Equal contribution.}

\date{} 

\begin{document}

\maketitle

\begin{abstract}
  \noindent We introduce the \emph{(Wishart) projection mechanism}, a randomized map of the form \(S \;\mapsto\; M f(S)\) with \(M \;\sim\;\mathsf{W}_d(1/r I_d,\, r)\) and study its differential privacy properties. For vector-valued queries $f$, we prove non-asymptotic DP guarantees without any additive noise, showing that Wishart randomness alone can suffice. For matrix-valued queries, however, we establish a sharp negative result: in the noise-free setting, the mechanism is not DP, and we demonstrate its vulnerability by implementing a near perfect membership inference attack (AUC $> 0.99$). We then analyze a noisy variant and prove privacy amplification due to randomness and low rank projection, in both large- and small-rank regimes, yielding stronger privacy guarantees than additive noise alone. Finally, we show that LoRA-style updates are an instance of the matrix-valued mechanism, implying that LoRA is not inherently private despite its built-in randomness, but that low-rank fine-tuning can be more private than full fine-tuning at the same noise level. Preliminary experiments suggest that tighter accounting enables lower noise and improved accuracy in practice.
\end{abstract}

\section{Introduction}\label{sec:intro}
Differential Privacy~(DP)~\citep{dwork2006calibrating} is widely regarded as the gold standard for protecting training data privacy in machine learning. Intuitively, DP limits the influence of any single example on the output, making it difficult to infer whether that example appeared in the training set. The most widely used DP algorithm in modern ML is DP-SGD~\citep{Abadi2016}, the private counterpart of the de facto standard optimization algorithm Stochastic Gradient Descent~(SGD). 

However, DP-SGD is computationally demanding and often incurs a substantial utility loss, especially for large models. While it remains one of the few viable choices for training machine learning models from scratch, in many practical deployments sensitive data enters primarily during fine-tuning, e.g. when an organization adapts a public pre-trained model on proprietary data. This motivates a simple strategy: start from a large public pre-trained model and enforce privacy only during fine-tuning. In the non-private setting, parameter-efficient fine-tuning (PEFT) \citep{han2024parameterefficient} updates only a small set of parameters while freezing the base model, substantially reducing memory and compute. This naturally raises the question: can PEFT similarly reduce the cost of DP fine-tuning?

\emph{Low-Rank Adaptation (LoRA)}~\citep{hu2022lora} is a widely used PEFT method that often matches full-parameter fine-tuning on downstream tasks~\citep{hu2022lora, dettmers2023qlora}. It freezes the pre-trained weights and inserts randomly initialised trainable low-rank matrices, dramatically shrinking the number of trainable parameters. Variants include adaptive-rank methods~\citep{zhang2023adalora}, quantization-aware tuning for low-bit backbones~\citep{dettmers2023qlora, li2023loftq}, stability/initialization refinements~\citep{hayou2024lora+, meng2024pissa}, and structural decompositions~\citep{liu2024dora}, each targeting better accuracy under tight compute and memory budgets. Approaches to privatising LoRA have also been proposed, including DP-LoRA~\citep{Liu2025}.

Several LoRA variants~\citep{sun2024improving, Hao2024Flora} already incorporate substantial randomness (e.g., repeated re-initialization of component weight matrices). Yet most privatisation algorithms treat LoRA updates as deterministic given the data and do not explicitly leverage this inherent randomness. At the same time, empirical studies report reduced memorisation under LoRA~\citep{hong2025evaluating} even without any explicit privatisation.

Following~\citet{Hao2024Flora}, we note that in certain LoRA variants (e.g. LoRA-FA~\citep{lorafa2023}, the update behaves as if it applies a random Wishart projection to the gradient. We formalize this behaviour through the following abstraction, (which allows us to rigorously analyse the privacy guarantee of these methods)

\begin{defn}[Projection mechanism]
    \label{defn:projection-mech}
    Let $\cS\subset\cX^n$  be a dataset collection and $f:\mathcal{S}\to\mathbb{R}^{n\times d}$ a query function.
    For $r\in\mathbb{N}$, the (Wishart) projection mechanism is defined by
    \[
    \mathcal{A}_{r}(S) \coloneqq f(S)M,
    \qquad
    M := Z^\top Z,
    \]
    where $Z\in\mathbb{R}^{r\times d}$ has i.i.d rows $z_k\overset{i.i.d}{\sim}\mathcal{N}\br{0,\frac{1}{r}I_d}$ or equivalently, $M\sim \mathsf{W}_d(\frac{1}{r}I_d,\,r)$\footnote{All of our analysis extends to Gaussian noise $Z\sim \cN(0, \sigma^2I_d)$ for any $\sigma^2 > 0$. We set $\sigma^2 = 1/r$ to match the usual normalization of Wishart random projection (e.g. in LoRA initializations). }.
\end{defn}
We therefore analyse $\mathcal{A}_{r}$ as a proxy for LoRA’s built-in randomness. Unlike the Johnson–Lindenstrauss~\citep{johnson1984extensions} transform, which is a random embedding that preserves norms, the Wishart projection is a random positive semidefinite reweighting that is isotropic in expectation but can distort norms in a single draw.

Recent work~\citep{malekmohammadi2025lowrankadaptationsecretlyimitates} conjectures that LoRA training dynamics may inherently possess privacy properties. Specifically, they show that individual rows of the gradient matrix asymptotically resembles the noise distribution of DP-SGD as the dimension $d\to \infty$. In contrast, we establish non-asymptotic differential privacy guarantees for the vector-valued projection mechanism in~\Cref{thm:vec-alt}. However, LoRA updates are matrix-valued: in our abstraction, they correspond to applying the projection mechanism to gradient matrices. We show that this matrix-valued projection mechanism, and hence LoRA, is not private in general~(\Cref{prop:neg-res}): in the setting without additional additive noise, neighbouring datasets can yield non-overlapping support, enabling (near) perfect membership inference with AUC $\geq 0.99$ as illustrated in~\Cref{tab:mia-noiseless}.
 
Nevertheless, LoRA’s built-in randomness is not irrelevant: it is simply insufficient on its own. We therefore consider a noisy variant of the projection mechanism, obtained by adding a small amount of additive noise. This enforces overlap for output distributions of neighbouring datasets, and the above attack no longer succeeds. In this setting we extend our guarantees from vectors to matrices and show that Wishart projection can amplify privacy. For fixed noise, composing with the projection yields strictly stronger privacy bounds than the noise level alone would suggest, both for large $r$~(\Cref{subsec:matproj-large-r}) and in the practically relevant small-$r$ regime~(\Cref{subsec:matproj-small-r}). A key implication is that low rank is beneficial beyond utility: at a fixed noise multiplier, projection can amplify privacy and yield smaller $\varepsilon$ than full fine-tuning, a point not made explicit in prior DP-LoRA analyses.

We summarize our contributions below. 
\begin{enumerate}[label=(\roman*)]
    \item We introduce the Wishart projection mechanism and prove non-asymptotic DP guarantees in the vector case. 
    \item We show the corresponding matrix mechanism is not DP, and demonstrate its vulnerability via a near-perfect MIA (AUC $> 0.99$).
    \item For a variant involving additive noise, we prove privacy amplification for both large and small ranks.
\end{enumerate}

\noindent\textbf{Organization.} In~\Cref{sec:prelim} we review DP preliminaries and give a brief overview of why LoRA is an instance of the Wishart projection mechanism.~\Cref{sec:proj-mechanism-vector} proves privacy for vector-valued functions and shows why it fails for matrices.~\Cref{sec:proj-mechanism-matrix} analyses the noisy Wishart projection mechanism for matrix-valued functions and establishes privacy amplification in both rank regimes. Finally,~\Cref{sec:discussion} covers related work, limitations, and open questions.
\section{Preliminaries}\label{sec:prelim}
We first review DP preliminaries, then introduce LoRA and show how it instantiates the projection mechanism.

Differential privacy limits how much the output distribution can change when a single data point is modified.
We say datasets $S, S'$ are neighbors, denoted $S \sim_H S'$, if they differ in a single entry (i.e., Hamming distance $d_H(S, S') = 1$).


\begin{defn}[Differential privacy]\label{defn:dp}
A randomised algorithm $\cA: \cS \to \cY$ is $(\varepsilon,\delta)$-DP if for all measurable $E\subseteq\cY$ and all $S\sim_H S'$,
\[
\Pr(\cA(S)\in E)\le e^\epsilon\Pr(\cA(S')\in E)+\delta,
\]
with probability taken over the internal randomness of $\cA$.
\end{defn}
Differential privacy is commonly enforced by additive perturbations: adding noise to the output of a non-private query \(f\), calibrated to its $\ell_2$ sensitivity, 
\[
\Delta \coloneqq \max_{S\sim_H S'} \norm{f(S)-f(S')}_2.
\]
A canonical example is the Gaussian mechanism: 
\begin{lem}[Gaussian Mechanism]
\label{lem:gaussian-mech}
Let $f: \mathcal{S} \to \mathbb{R}^d$ be a function with $\ell_2$-sensitivity $\Delta$. For any $\epsilon, \delta \in (0, 1)$, the Gaussian mechanism $\mathcal{A}(S) = f(S) + Z$ is $(\epsilon, \delta)$-DP, for \[Z \sim \mathcal{N}(0, \sigma^2 I_d), \text{ where }
\sigma \ge \frac{2\Delta \sqrt{ \log(1.25/\delta)}}{\epsilon}.\]
\end{lem}


Restricting $\cA$ to datasets in $\cD \subseteq \cS$ gives DP guarantees \emph{conditioned on $\cD$}. This often reduces $\Delta$ and improves utility, but offers no privacy outside $\cD$. One way to obtain full privacy is to privately check whether the dataset belongs to $\cD$, e.g., via Propose-Test-Release~\cite{dwork2009differential}.


\paragraph{LoRA is a Projection Mechanism}
\label{subsec:lora-priv}

Low-Rank Adaptation (LoRA;~\citet{hu2022lora}) is one of the most popular parameter-efficient fine-tuning approaches for Large Language Models. It augments a pretrained weight matrix $W_0\in\reals^{n\times d}$ by a \emph{low-rank} update:
\[
W = W_0 + BA, \; B\in\R^{n\times r}, \;A\in\reals^{r\times d}, \; r\ll\min\bc{d,n}.
\]
By freezing $W_0$ and only optimizing the smaller matrices $(B, A)$, LoRA reduces computational cost while preserving the base model's capabilities. 

LoRA-FA\citep{lorafa2023} is a variant that fixes $A$ and only updates $B$. When $A$ is initialized Gaussian and then \emph{frozen}, we update $B$ by
\[
B_{t+1}\;=\;B_t - \eta\,(\nabla_W \cL(W_t))A^\top.
\]
Then after $T$ steps, we can write
\begin{equation}\label{eq:weight-final}
W_T \;=\; W_0 - \eta \sum_{t=1}^T \big(\nabla_W \cL(W_t)\big)\,(A^\top A).
\end{equation}
Thus each step uses a gradient \(\nabla_W \cL(W_t) \in \R^{n \times d}\) that is \emph{right-projected} by the random matrix $A^\top A$ (a rank-$r$ Wishart). Right multiplying a vector or matrix by a randomly sampled Wishart is precisely the projection mechanism studied in this work. Prior work~\citep{malekmohammadi2025lowrankadaptationsecretlyimitates} suggested that this random projection structure may yield DP guarantees. However, as we show in \Cref{sec:proj-mechanism-vector}, the matrix-valued projection mechanism is not DP in general (see \Cref{prop:neg-res}).

However, our results in~\Cref{sec:proj-mechanism-matrix} show that noisy versions of the projection mechanism enjoy $(\varepsilon, \delta)$-DP guarantees that in certain regimes are stronger than simple noise addition (i.e., Gaussian mechanism) would suggest due to the randomness stemming from the projection.

\section{Projection Mechanism: Privacy and Limits}\label{sec:proj-mechanism-vector}
We study the privacy properties of the projection mechanism~(\Cref{defn:projection-mech}). We first analyze the special case of vector-valued queries \(f:\cS\to\reals^d\) and show in Section~\ref{subsec:privacy-guarantee-vector} that it is differentially private under an alignment assumption. In contrast, for matrix-valued queries we prove a negative result: \Cref{prop:neg-res} shows the privacy loss is unbounded (\(\epsilon=\infty\) for any \(\delta<1\)). This also implies that LoRA-FA and LoRA is not private. 

\subsection{Privacy Guarantees for Vector-Valued Functions}\label{subsec:privacy-guarantee-vector}
Let \(f:\cX^n\to\reals^d\) be a query with \(\|f(S)\|_2=1\) for all \(S\); this normalization can always be enforced by scaling, and we assume it throughout. For a dataset collection \(\cD\subset\cX^n\), define the \emph{minimum alignment}
\begin{equation}
\label{eq:min-alignment-vec}
\rho(\cD,f)\coloneqq \min_{\substack{S,S'\in\cD\\ S\sim_H S'}} f(S)^\top f(S') \in [-1,1].
\end{equation}
When the context is clear, we write \(\rhoalign:=\rho(\cD,f)\). Large $\rhoalign$ means that neighbouring query outputs are nearly co-directional. In this regime the laws of $Mf(S)$ and $Mf(S')$ are harder to distinguish, leading to tighter $(\varepsilon_\rho,\delta_\rho)$ guarantees~(\Cref{fig:vector-case-simulation}). Crucially, much like a sensitivity parameter in additive mechanisms, $\rhoalign$ is a property fixed by the dataset collection \(\cD\) and the query $f$ and is not controlled by the projection mechanism itself. One can also preprocess the data to improve the alignment $\rho$ through e.g. noise addition. See Appendix~\ref{app:vecprivamplification-and-applications} for details. 

\begin{table*}[t]
\centering
\small
\setlength{\tabcolsep}{6pt}
\renewcommand{\arraystretch}{1.15}
\begin{tabular}{llccccc}
\toprule
\multirow{2}{*}{\textbf{Setting}} & \multirow{2}{*}{\textbf{Metric}} &
\multicolumn{5}{c}{\textbf{Rank $r$}} \\
\cmidrule(lr){3-7}
 & & \textbf{16} & \textbf{32}  & \textbf{128} & \textbf{256} & \textbf{512} \\
\midrule
\multirow{2}{*}{Train from scratch} 
& Test Acc (\%) & 57.31 & 61.76     & 66.36 & 67.34 & 68.36 \\
& AUC (\%)      & 99.86 & 99.64     & 100.00 & 100.00 & 100.00 \\
\midrule
\multirow{2}{*}{Pretrain+Finetune } 
& Test Acc (\%) & 75.98 & 75.63  & 76.12 & 76.04 & 76.45 \\
& AUC (\%)      & 99.81 & 99.98  & 100.00 & 100.00 & 100.00 \\
\bottomrule
\end{tabular}
\vspace{1mm}
\caption{Membership inference attack performance for noise-free LoRA under a static-poison canary construction.}
\label{tab:mia-noiseless}
\vspace{-5mm}
\end{table*}

Let \(t_{\ell}(\cdot)\) and \(\kappa_{\ell}(\cdot)\) denote the quantile functions of the Student-\(t\) and \(\chi^2\) distributions, respectively, each with \(\ell\) degrees of freedom.
    
\begin{restatable}{thm}{vecprivacynew}\label{thm:vec-alt}
    For a dataset collection $\cD$ and a query function $f$ with outputs in $\bR^d$, let $\rho > 0$ be the minimum alignment for $f$ on $\cD$ as defined in~\Cref{eq:min-alignment-vec}. Let $\delta' > 0$. If $\rho > \frac{t_r(1-\delta')}{\sqrt{r + t_r(1-\delta')^2}}$, then, the projection mechanism~(\Cref{defn:projection-mech}) with rank $r$ is $(\varepsilon_\rho, \delta_\rho)$-DP on $\cD$, with
    \[
    \begin{aligned}
    \delta_\rho &=
      \bE_{x\sim \cX_r^2} \bs{\Phi\br{-\frac{\rho \sqrt{x}}{\sqrt{1-\rho^2}}}} + 3\delta'
    \end{aligned}
    \]
    \[
    \begin{aligned}
     \varepsilon_\rho  \leq \frac{d - r + 1}{2} \ln \br{\rho + K} + \frac{(1-\rho + K)\kappa_{d + r -1}(1-\delta')}{2(\rho - K)}
    \end{aligned}
    \]
    where $K = \sqrt{\frac{1-\rho^2}{r}}t_r(1-\delta')$.
\end{restatable}


\begin{figure}
    \vspace{-5pt}
    \centering
    \includegraphics[width=0.55\linewidth]{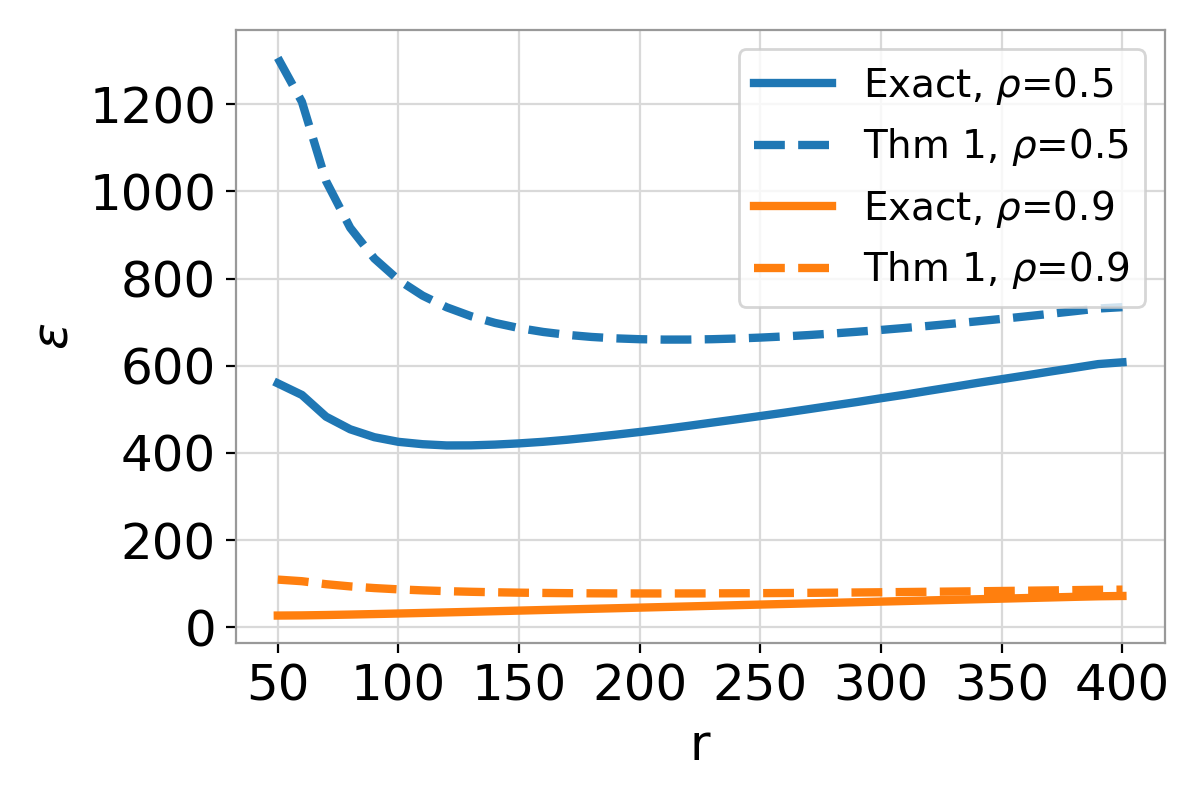}
    \caption{Privacy loss $\varepsilon$ v.s. rank $r$ at $\delta = 0.01$ and $d = 400$: Monte Carlo simulation of exact privacy profile and~\Cref{thm:vec-alt}.}
    \label{fig:vector-case-simulation}
    \vspace{-15pt} 
\end{figure}
Here $r$ is the rank of the Wishart matrix $M=ZZ^\top$ with $Z\in\mathbb{R}^{d\times r}$, and is an algorithmic choice.
As illustrated in~\Cref{fig:vector-case-simulation}, for fixed $\delta$ the resulting $\varepsilon$ depends non-monotonically on $r$. This reflects two competing effects: smaller $r$ makes $M$ lower-rank, confining the output to a lower-dimensional subspace and potentially reducing leakage via compression; but it also relies on fewer Gaussian directions, increasing tail risk and potentially weakening the \((\varepsilon,\delta)\) guarantee.

\noindent\textbf{On the dependence of $\rhoalign$ on $n$.}
The minimum alignment $\rhoalign$ measures how much the direction of the normalised query can change under a single record change. For average-type queries, one replacement perturbs the unnormalised vector by $O(1/n)$, which typically yields $\rhoalign = 1-\widetilde{O}(1/n^2)$. Concretely, let $\tilde f(S)=\frac{1}{n}\sum_{i=1}^n g(x_i)$ with $\|g(x)\|_2\le L$ and assume $\inf_{S\in\cD}\|\tilde f(S)\|_2\ge c_0>0$. Then for neighboring datasets $S, S'$, we have $\|\tilde f(S)-\tilde f(S')\|_2\le 2L/n$. Defining $f(S)\coloneqq \tilde f(S)/\|\tilde f(S)\|_2$ gives $\|f(S)-f(S')\|_2\le 4L/(c_0 n)$ and hence $\rhoalign \ge 1-\frac{8L^2}{c_0^2 n^2}$, using $1-u^\top v=\tfrac12\|u-v\|_2^2$ for unit vectors.
The same argument applies to mini-batch gradients: for $\tilde f(B)=\frac{1}{|B|}\sum_{x\in B}\nabla \ell(\theta;x)$, if $\|_2\nabla \ell(\theta;x)\|_2\le L$ and $\|\tilde f(B)\|_2$ is bounded away from zero, then $\rhoalign = 1-\widetilde{O}(1/|B|^2)$.

\subsection{Negative Results for Matrix-Valued Functions}\label{subsec:matrix-negative-result}

\Cref{thm:vec-alt} establishes \((\varepsilon,\delta)\)-DP for the projection mechanism on \emph{vector}-valued queries under an alignment condition. Since the LoRA-FA update can be cast as the same mechanism with \emph{matrix}-valued outputs (see~\Cref{sec:prelim}), one might expect an analogous guarantee for matrices and thus privacy for LoRA-FA. We show this is false: without additive noise, the matrix-valued projection mechanism is not DP.

\begin{proposition}\label{prop:neg-res}
    Let $f:\cS\to \bR^{n\times d}$ be a non-trivial function and $\cA(S)$ be the projection mechanism~(\Cref{defn:projection-mech}) with query function $f$. Then, there exists two neighboring datasets $S, S' \in \cS$ and a measurable event $E$ with 
    \[
    \begin{aligned}
        \Pr(\cA(S)\in E) = 1,\quad
        \Pr(\cA(S')\in E)= 0
    \end{aligned}
    \]
    Consequently, $\cA$ is not $(\varepsilon, \delta)$-DP for any $\varepsilon$ and any $\delta <1$.
\end{proposition}
\noindent\textbf{Proof intuition: }
Let \(V:=f(S)\) and \(V':=f(S')\). As $V\neq V'$ such that $V - V'\neq 0$ and $\text{rank}(V - V') = s \geq 1$. Given $S$ the mechanism outputs the random variable \(Y=VM =  VZ Z^\top\), where each column $z_i$ of $Z$ is i.i.d. $\cN(0, I_d/r)$. For the two outputs to coincide we require $(V - V') M = 0$, i.e. $ (V - V') Z= 0$: every random direction $z_i$ must be orthogonal to the rowspace of $(V - V') $. As $\text{rank}(V - V') = s$, its orthogonal complement is a $(d - s)$-dim subspace. A continuous Gaussian vector $z_i$ lands in a fixed lower dimensional subspace with probability 0. Hence, $VM \neq V'M$ almost surely. Taking $E = \text{Supp}(VM)$ concludes the proof. (See~\Cref{lem:image-overlap-measure-zero}).

\Cref{prop:neg-res} extends beyond our abstract projection mechanism to directly rule out “privacy for free” arising from LoRA’s random initialization. Specifically, the first update step of standard LoRA (where $B_0 = 0$) is equivalent to the~\Cref{defn:projection-mech}, implying standard LoRA cannot satisfy DP without additional additive noise (See Appendix~\ref{app:negative-result-lora} for detailed explanation).

Motivated by the negative result above, we test noise-free LoRA against a membership inference attack (MIA). We run 2000 independent trials on CIFAR-10 using the popular static-poison canary construction~\cite{nasr2021adversary} in two regimes: (i) training from scratch, and (ii) pretraining followed by fine-tuning on CIFAR-10. \Cref{tab:mia-noiseless} reports test accuracy and ROC-AUC. Across both regimes and all ranks, the attack achieves near-perfect separability (AUC $\approx 1$), consistent with the non-overlap intuition suggested by the theory. See Appendix~\ref{app:exp-mia} for details. 

These findings do not conflict with prior work showing that LoRA can reduce MIA success~\citep{malekmohammadi2025lowrankadaptationsecretlyimitates,hong2025evaluating}, since MIA performance depends strongly on the threat model. We consider a strong adversary who knows all but one training records and can choose an adversarial canary, mirroring DP’s worst-case auxiliary-information setting where the guarantee must hold for any choice of differing example. Our work shows that under a classical DP style adversary, LoRA is not private. 

\section{The Noisy Projection Mechanism}\label{sec:proj-mechanism-matrix}

As shown in the previous section, the symmetry constraint induced by a Wishart (PSD) projection can make the output supports of neighboring matrix inputs disjoint. This motivates adding a small amount of Gaussian noise to enforce overlap if we want to keep DP guarantee. 

In this section, we show that once overlap exists, the intrinsic randomness of the Wishart projection can
amplify privacy beyond what an analysis based on additive noise alone would suggest.
We study two noisy variants (\ref{eq:M1} and~\ref{eq:M2}), and distinguish two complementary amplification mechanisms:
(i) a large-rank phenomenon~(\Cref{subsec:matproj-large-r}) driven by posterior-stable residual randomness in $M$;
and (ii) a small-rank phenomenon~(\Cref{subsec:matproj-small-r}) given by random subspace hiding.

\noindent\textbf{Setup and Mechanism. }
Let \(d,n,r\ge 1\) and fix the noise scale \(\sigma_G>0\). Draw \(Z\in\mathbb{R}^{d\times r}\) with i.i.d.\ entries \(Z_{ik}\sim\mathcal{N}(0,1/r)\) and set \(M:=ZZ^\top\in\mathbb{R}^{d\times d}\). Let \(\Xi \coloneqq [\xi_1,\ldots,\xi_n]\in\mathbb{R}^{d\times n}\) with i.i.d columns \(\xi_k\sim\mathcal{N}(0,\sigma_G^2 I_d)\).
\begin{defn}[Noisy Projection Mechanisms]
Given $V=\bs{v_1,\dots,v_n}\in\bR^{d\times n}$, define\footnote{We adopt left-multiplication in \ref{eq:M1} and \ref{eq:M2} for analytical convenience, but they have the identical privacy guarantee as~\Cref{defn:projection-mech}. }:
\begin{align}
\cA_1(V) &\coloneqq MV+\Xi \tag{M1}\label{eq:M1}\\
\cA_2(V) &\coloneqq M(V+\Xi) \tag{M2}\label{eq:M2}
\end{align}
\end{defn}

\subsection{Matrix Projection in the large $r$ regime}\label{subsec:matproj-large-r}

We first study the mechanism \(\cA_1(V)=MV+\Xi\)~(\ref{eq:M1}). Unlike the vector case, privacy can fail even when only one column of $V$ changes. The issue is that the same random matrix \(M\) is reused across columns: after observing \(n-1\) columns, an adversary can learn information about \(M\) and use it to infer the remaining column more accurately. Thus additive noise \(\Xi\) is needed to mask this leakage and ensure distributional overlap.

\subsubsection{Setting and Main Theorem} We first prove a privacy guarantee for a one-column adjacency, and then lift it to a more general matrix adjacency by composing such one-columns steps. Before stating our results, we define these adjacency relations in~\Cref{def:adj-matrix}.

Throughout this section we assume all columns are unit vectors:
\(\|v_j\|_2=\|v'_j\|_2=1\).  For \(V,V'\in\R^{d\times n}\), write
\(J(V,V') \coloneqq \{j\in[n]: v_j\neq v'_j\}\), and for any
\(j\in[n]\), let \(V_{-j}\) denote \(V\) with column \(j\) removed.

\begin{defn}[Adjacency notion]
\label{def:adj-matrix}
Fix parameters \(\beta_\perp\in[0,1]\) and
\(\rho_\parallel,\rho_\perp\in[-1,1]\).
\begin{enumerate}[leftmargin=*, itemsep=2pt, topsep=2pt, parsep=0pt, partopsep=0pt]
\item \textbf{(One-column adjacency).} Fix \(j\in[n]\) and set \(S\coloneqq \mathrm{span}(V_{-j})\subset\R^d\).
We write \(V\sim_A^{(j)}V'\) if \(V_{-j}=V'_{-j}\) and, writing
\(v_j=a+b\) and \(v'_j=a'+b'\) with \(a,a'\in S\) and \(b,b'\in S^\perp\), we have
\[
\norm{b}_2,\norm{b'}_2=\beta_\perp,~
\frac{\ip{a}{a'}}{\norm{a}_2\norm{a'}_2}\ge \rho_\parallel,~
\frac{\ip{b}{b'}}{\norm{b}_2\norm{b'}_2}\ge \rho_\perp.
\]

\item \textbf{(Matrix adjacency).} We write \(V\sim_A V'\) if there
exists an ordering \(j_1,\dots,j_k\) of \(J(V,V')\) where \(k = \abs{J(V,V')} \),
matrices \(V^{(0)},\dots,V^{(k)}\) such that \(V^{(0)}=V\),
\(V^{(k)}=V'\), and \(V^{(t-1)}\sim_A^{(j_t)} V^{(t)}\) for all
\(t\in[k]\).
\end{enumerate}
\end{defn}

The parameter \(\beta_\perp\) measures how much of the changed column
lies outside the span exposed by the other columns. The parameters
\(\rho_\parallel\) and \(\rho_\perp\) control alignment within \(S\)
and \(S^\perp\), respectively.

\begin{thm}[Privacy of~\ref{eq:M1} in the large-$r$ regime]
\label{thm:matrix-large-r-main}
Fix \(\beta_\perp,\rho_\parallel,\rho_\perp\) and let \(V\sim_A V'\) with \(k=|J(V,V')|\). Define \((\varepsilon_{\mathrm{one}},\delta_{\mathrm{one}})\)
as the privacy parameters for a single admissible step \(V\sim_A^{(j)}V'\)
(given in \Cref{thm:matrix-onecol-directional}).
Then \(\cA_1\) is \((\varepsilon,\delta)\)-DP for inputs \(V,V'\), with
\[
\varepsilon \le k\,\varepsilon_{\mathrm{one}},
\qquad
\delta \le k\,e^{(k-1)\varepsilon_{\mathrm{one}}}\,\delta_{\mathrm{one}}.
\]
\end{thm}

\subsubsection{Proof Sketch}
Let \(V^{(0)},\dots,V^{(k)}\) be the adjacency chain from \Cref{def:adj-matrix}. As in group privacy, it suffices to prove a privacy guarantee for a single step where only one column changes, and then use a group-privacy style analysis to compose those guarantees along the chain. So the main work is the one-column case

Fix a step \(V\sim_A^{(j)}V'\). The matrices agree on all columns except $j$ so \(V_{-j}=V'_{-j}\). Write the mechanism's output as:
\[
(X,Y)\coloneqq \bigl(Mv_j+\xi_j,\; MV_{-j}+\Xi_{-j}\bigr).
\]
Let \(S=\mathrm{span}(V_{-j})\) with orthonormal basis \(U\), and set \(G\coloneqq U^\top Z\).
Using the decomposition \(M=M_\parallel+M_\perp\) from Appendix~\ref{app:matrix-large-r}, we have:
(i) \(M_\perp V_{-j}=0\), so \(Y\) depends only on \(M_\parallel\); and
(ii) \emph{posterior stability}: conditional on \(G\), the residual block \(M_\perp\) remains independent of \(Y\).
This is the key decoupling: after observing the unchanged columns, the randomness in \(M_\perp\) is still ``fresh''.

Decompose \(v_j=a+b\) and \(v_j'=a'+b'\) with \(a,a'\in S\) and \(b,b'\in S^\perp\).
Since \(M_\perp a=0\), the released column splits as
\[
Mv_j+\xi_j
=\underbrace{(M_\parallel v_j+\xi_j)}_{=:C(v_j)}
+\underbrace{M_\perp b}_{=:R(b)} .
\]
We bound these two parts separately and then compose.
\noindent\textbf{Residual term. }For the residual term \(R(b)\), posterior stability~(\Cref{lem:posterior-stable}) lets us apply the \emph{vector-valued} DP analysis from \Cref{sec:proj-mechanism-vector}, yielding an \((\varepsilon_\perp,\delta_\perp)\) bound. Let
\(\Phi\) be the standard normal CDF; let \(t_q\) and \(\kappa_\ell\)  be the quantiles of Student $t$ and $\chi_\ell^2$ distributions, respectively. 

\begin{restatable}[Residual DP
bound]{lem}{privrescomponent}\label{lem:resid-dp} 

Write \(s\coloneqq \dim(S)\) and \(m\coloneqq d-s\). Fix
\(\rho_\perp\in(0,1]\) and \(\delta'\in(0,1)\), and define
\[
K_\perp = \sqrt{\tfrac{1-\rho_\perp^2}{q}}t_q\br{1-\nicefrac{\delta'}{3}}, 
\]
Then the residual release \(R(b)=M_\perp b\) is
\((\varepsilon_\perp,\delta_\perp)\)-DP for inputs \(b,b'\)
satisfying~\Cref{def:adj-matrix} where
\[
\delta_\perp
=\E_{X\sim\chi^2_q}\bs{\Phi\br{-\tfrac{\rho_\perp\sqrt{X}}{\sqrt{1-\rho_\perp^2}}}}+\delta',\]\[\varepsilon_\perp
\le
\tfrac{m-q+1}{2}\ln(\rho_\perp+K_\perp)
+\frac{\kappa_{m+q-1}\br{1-\nicefrac{\delta'}{3}}}{2}\left(\tfrac{1}{\rho_\perp-K_\perp}-1\right).  \] 
\end{restatable}

\noindent\textbf{Correlated term. }For the correlated term \(C(v_j)\), we apply the Gaussian mechanism with a high-probability directional sensitivity bound for \(M_\parallel\):
with probability at least \(1-\delta_s\) over \(Z\),
\(
\|M_\parallel(v_j-v'_j)\|_2 \le \Gamma_{\delta_s}\|v_j-v'_j\|_2
\)
(\Cref{lem:dir-bound}).
Conditioning on this event we get the correlated term is \(
(\varepsilon_\parallel,\delta_\parallel)-DP
\)
with 
\begin{equation}\label{eq:eps-parallel-mat}    
\varepsilon_\parallel
\coloneqq
\frac{\Gamma_{\delta_s}\Delta_v}{\sigma_G}\sqrt{2\ln\br{\nicefrac{1.25}{\delta_\parallel}}}.
\end{equation}

Finally, sequential composition gives
\(
\varepsilon_{\mathrm{one}}\le \varepsilon_\parallel+\varepsilon_\perp
\)
and
\(
\delta_{\mathrm{one}}\le \delta_\parallel+\delta_\perp+\delta_s.
\)
\begin{restatable}{thm}{privmatrixprojonecol}\label{thm:matrix-onecol-directional}
Fix $j\in[n]$. Let $V\sim_A^{(j)}V'$ with parameters $(\rho_\parallel,\beta_\perp,\rho_\perp)$ defined in~\Cref{def:adj-matrix}.
Fix $\delta_\parallel,\delta_s\in(0,1)$, and choose $\delta'>0$ as in \Cref{lem:resid-dp} to obtain $(\varepsilon_\perp,\delta_\perp)$.
Let $p=\mathrm{rank}(U^\top Z)$ and define $\Gamma_{\delta_s}$ as
above, and \(\varepsilon_\parallel\) as in \Cref{eq:eps-parallel-mat}.
Then the mechanism $\cA_1(V)=MV+\Xi$ is $(\varepsilon_{\mathrm{one} },\delta_{\mathrm{one} })$-DP for inputs $V\sim_A^{(j)}V'$, with
\[
\varepsilon_{\mathrm{one} } \le \varepsilon_\parallel+\varepsilon_\perp,
\qquad
\delta_{\mathrm{one} }\le \delta_\parallel+\delta_\perp+\delta_s.
\]
\end{restatable}
Composing the one-column bounds along the chain yields \Cref{thm:matrix-large-r-main}.

\subsubsection{Comparison with the Gaussian Mechanism}\label{subsec:compare-classical}
We next compare the privacy guarantee in \Cref{thm:matrix-large-r-main} to a classical baseline that ignores the randomness in $M$: it treats $V\mapsto MV$ as deterministic and applies the Gaussian mechanism calibrated to the Frobenius sensitivity.

Fix \(V,V'\in\mathbb{R}^{d\times n}\). Let \(J=J(V,V')\) with \(k \coloneqq |J|\). Assume \(\|v_j\|_2=\|v'_j\|_2=1\) and \(\|v_j-v'_j\|_2\le \Delta_v\) for all \(j\in J\), then
\begin{equation}
\label{eq:DV-frob-cmpr}
\|\Delta V\|_F^2=\sum_{j\in J}\|v_j-v'_j\|_2^2 \le k\Delta_v^2.
\end{equation}
A deterministic sensitivity bound gives \(\|MV-MV'\|_F\le \|M\|\,\|\Delta V\|_F\). Using \(\|M\|\approx \sigma_M^2(\sqrt d+\sqrt r)^2\), we obtain
\begin{equation}
\label{eq:eps-class-mat-clean}
\varepsilon_{\mathrm{class}}^{(k)}
\;\le\;
\sigma_M^2(\sqrt d+\sqrt r)^2\,\sqrt{k}\,K_G,
\end{equation}
where $K_G \coloneqq \frac{\Delta_v}{\sigma_G}\sqrt{2\ln\!\Bigl(\frac{1.25}{\delta_G}\Bigr)} .$ By \Cref{thm:matrix-large-r-main}, the $k$-column guarantee composes as
\begin{equation} \label{eq:eps-ours-k-clean}
\varepsilon_{\mathrm{ours}}^{(k)} = k(\varepsilon_\parallel+\varepsilon_\perp),
\text{ with }  \varepsilon_\parallel \le \sigma_M^2(\sqrt d+\sqrt p)\sqrt p\,K_G,   
\end{equation}
where \(p\le s\) is the effective dimension of the revealed column-span seen through the random projection (formally \(p=\mathrm{rank}(U^\top Z)\)), and \(s=\dim(S)\) with \(S=\mathrm{span}(V_{-j})\).

For comparison, we set \(\delta_\parallel=\nicefrac{\delta_G}{3}\). Combining~\Cref{eq:eps-class-mat-clean} and~\Cref{eq:eps-ours-k-clean} yields
\begin{equation}
\label{eq:ratio-decomp-clean}
\frac{\varepsilon_{\mathrm{ours}}^{(k)}}{\varepsilon_{\mathrm{class}}^{(k)}}
\;\le\; \sqrt{k} \left(
\underbrace{\frac{(\sqrt d+\sqrt p)\sqrt p}{(\sqrt d+\sqrt r)^2}}_{\text{Gaussian part}}
\;+\;
\underbrace{\frac{\varepsilon_\perp}{\sigma_M^2(\sqrt d+\sqrt r)^2\,K_G}}_{\text{residual part}} \right).
\end{equation}

We can further bound $\varepsilon_\perp$ by~\Cref{lem:resid-dp}, 
\[
\varepsilon_\perp
\le
\underbrace{\frac{m-q+1}{2}\ln(\rho_\perp+\alpha)}_{T_1}
\;+\;
\underbrace{\frac{K}{2}\left(\frac{1}{\rho_\perp-\alpha}-1\right)}_{T_2}.
\]
where \(m=d-s\), \(q=r-p\) and \(\alpha\) and \(K\) are the \((1-\delta'/3)\)-quantiles defined in \Cref{lem:resid-dp}.

In particular, when \(\rho_\perp\) is close to \(1\) and \(\alpha\) is small (e.g., for larger \(q\) and moderate \(\delta'\)), the residual part is small, so the comparison is dominated by the Gaussian part. In the rank-deficient regime $s<r$, we moreover have $p=s$ a.s., so the Gaussian coefficient simplifies to
\begin{equation}
\label{eq:ratio-gauss-s<r-clean}
\frac{(\sqrt d+\sqrt s)\sqrt s}{(\sqrt d+\sqrt r)^2}.
\end{equation}
In this regime, the coefficient is $<1$, so our bound improves on the classical baseline. 
\subsection{Matrix Projection in the small $r$ regime}\label{subsec:matproj-small-r}
The small-\(r\) regime yields privacy amplification through a
different mechanism from the large $r$ regime. Here we analyse
\(\cA_2(V)=M(V+\Xi)\)~(\ref{eq:M2}). One might naively treat \(M\) as
post-processing: add Gaussian noise \(\Xi\) calibrated to the
Frobenius sensitivity~\(\|\Delta V\|_F\), then multiply by \(M\).
However, the important observation is that the output of $\cA_2$ always lies in the random column space \(\mathrm{col}\br{M}=\mathrm{col}\br{Z}\).
Consequently, conditioned on \(M\), privacy depends only on the
projected difference \(P_M\Delta V\), where \(P_M\) is the orthogonal
projector onto \(\mathrm{col}\br{M}\). When \(r\ll d\) the
projection typically captures only a small fraction of
\(\|\Delta V\|_F\), reducing the effective sensitivity. In rare
alignment events where \(\Delta V \approx P_M \Delta V\) the
amplification disappears and the guarantee becomes comparable to the naive approach.

\subsubsection{Main Theorem}

For \(V,V'\in\R^{d\times n}\) write \(\Delta V\coloneqq V-V'\).
For \(\varepsilon>0\) and \(\mu\ge 0\), define the Gaussian two-sided
tail function
\[
T(\varepsilon;\mu)
\coloneqq
\Phi\!\left(\frac{-\varepsilon+\mu/2}{\sqrt{\mu}}\right)
+
1-\Phi\!\left(\frac{\varepsilon+\mu/2}{\sqrt{\mu}}\right),
\]
where \(\Phi\) is the standard normal CDF and
\(T(\varepsilon;0)=0\).

\begin{restatable}[Privacy of \ref{eq:M2} in the small $r$-regime]{thm}{smallrmatrixdpprojmech}
\label{thm:matrix-dp-proj-mechanism-small-r}
Fix \(V,V'\in\R^{d\times n}\) and write \(s=\min\br{n,d}\). Then for
any \(\alpha\in(0,1)\) and any \(\varepsilon>0\), the mechanism
\(\cA_2\) is \(\br{\varepsilon,\delta_\alpha(\varepsilon)}\)-DP with
\[
\delta_\alpha(\varepsilon)
\le T\br{\varepsilon;\frac{\alpha\norm{\Delta V}_F^{2}}{\sigma_G^{2}}}
+
s\bs{1 - I_{\alpha}\br{\frac{r}{2},\frac{d-r}{2}}},
\]
where \(I_\alpha(a,b)\) is the regularized incomplete Beta function
(i.e.\ \(I_\alpha(a,b)=\Pr\bs{B\le \alpha}\) for
\(B\sim\mathrm{Beta}(a,b)\)).
\end{restatable}

\subsubsection{Proof Sketch}

Fix \(V,V'\) and condition on \(M\). Let \(P_M\) denote the orthogonal
projector onto \(\mathrm{col}(M)\). Conditioned on \(M\), the output $M(V + E)$
is a Gaussian matrix but supported on the rank-\(r\) subspace
\(\mathrm{col}(M)\). Consequently, the privacy loss depends on
\(\Delta V\) only through its projection onto \(\mathrm{col}(M)\),
namely \(\norm{P_M\Delta V}_F\). Let \(P\br{\cdot\mid M}\) and \(Q\br{\cdot\mid M}\) be the conditional
output laws under inputs \(V\) and \(V'\), respectively. Since
$\sigma_G > 0$ we have $P\br{\cdot \mid M} \ll Q\br{\cdot\mid M}$, so
we can define the conditional privacy loss random variable
\[
L_M(Y)\;\coloneqq\;\log\frac{dP\br{\cdot\mid M}}{dQ\br{\cdot\mid M}}(Y),
\quad Y\sim P\br{\cdot\mid M}.
\]
Then as shown in \Cref{lem:conditional-priv-loss} we have,
\[
L_M(Y)\sim
\cN\br{-\tfrac{\norm{P_M\Delta V}_F^2}{2\sigma_G^2},
\tfrac{\norm{P_M\Delta V}_F^2}{\sigma_G^2}}.
\]
This motivates the following good-projection event
\[
\cG_\alpha
\coloneqq
\bc{M: \norm{P_M\Delta V}_F^2\le \alpha\norm{\Delta V}_F^2},
\quad 0<\alpha<1.
\]

On the good event \(\cG_\alpha\), $L_M(Y)$ has a Gaussian tail:
\begin{equation}
\label{eq:gaussian-tail-small-r}
\Pr\br{\abs{L_M(Y)}>\varepsilon \mid M}
\le
T\br{\varepsilon;\tfrac{\alpha\norm{\Delta V}_F^2}{\sigma_G^2}}.
\end{equation}
Thus the unconditional $\delta$ is the Gaussian tail on $\cG_\alpha$ plus the failure probability  \(\Pr\br{\cG_\alpha^c}\). It remains to bound \(\Pr\br{\cG_\alpha^c}\).

Let \(\text{rank}(\Delta V) = s\), and let \(u_1,\dots,u_{s}\in\R^d\) be an orthonormal basis for \(\mathrm{col}\br{\Delta V}\). If \(\norm{P_M u_i}_2^2\le \alpha\) for all \(i\in\bs{s}\), then
\(
\norm{P_M\Delta V}_F^2 \le \alpha\norm{\Delta V}_F^2
\). Thus, the bad event \(\cG_\alpha^c\) implies \(\norm{P_M u_i}_2^2>\alpha\) for some
\(i\). Hence, by a union bound,
\[
\Pr(\cG_\alpha^c)
\le
 \sum_{i=1}^s \Pr\br{\norm{P_M u_i}_2^2>\alpha}.
\]
Since \(M=ZZ^\top\) with \(Z\in\R^{d\times r}\) Gaussian, \(\mathrm{col}(M)=\mathrm{col}(Z)\) is a random Haar rank-\(r\) subspace of \(\R^d\), and \(P_M\) is the orthogonal projector onto this
subspace~(\Cref{lem:delta-M-bound}).  For a fixed unit vector \(u\in\R^d\), the captured energy
\(\|P_Mu\|_2^2\) then follows the Beta law \(\mathrm{Beta}\!\br{\tfrac
r2,\tfrac{d-r}{2}}\). Therefore,
\[
\Pr\br{\cG_\alpha^c}
\le
s\bs{1 - I_{\alpha}\br{\frac{r}{2},\frac{d-r}{2}}}.
\]
Combining this with \eqref{eq:gaussian-tail-small-r} yields \Cref{thm:matrix-dp-proj-mechanism-small-r}.

\subsubsection{Comparison with the Gaussian Mechanism}
Without conditioning on \(M\), the standard Gaussian analysis uses sensitivity \(\|\Delta V\|_F\) and gives, for any $\varepsilon>0$,
\begin{equation}\label{equ:gauss-delta-small-r}
\delta_{\mathrm{Gauss}}(\varepsilon)
\;=\; T\!\left(\varepsilon;\frac{\|\Delta V\|_F^2}{\sigma^2}\right).
\end{equation}
This is robust but pessimistic, as it treats the full gap $\|\Delta V\|_F$ as visible. Conditioning on $M$ yields a tighter bound: \Cref{thm:matrix-dp-proj-mechanism-small-r} splits $\delta$ into a Gaussian term on the good event \(\mathcal{G}_\alpha\) and a failure probability:
\[
\delta_{\mathrm{ours}}(\varepsilon)
\;\le\;
T\!\left(\varepsilon;\,\frac{\alpha\|\Delta V\|_F^2}{\sigma^2}\right)
\;+\;
s\!\left[1 - I_{\alpha}\!\left(\frac{r}{2},\,\frac{d-r}{2}\right)\right].
\]
\begin{figure}
    \centering
    \includegraphics[width=0.6\linewidth]{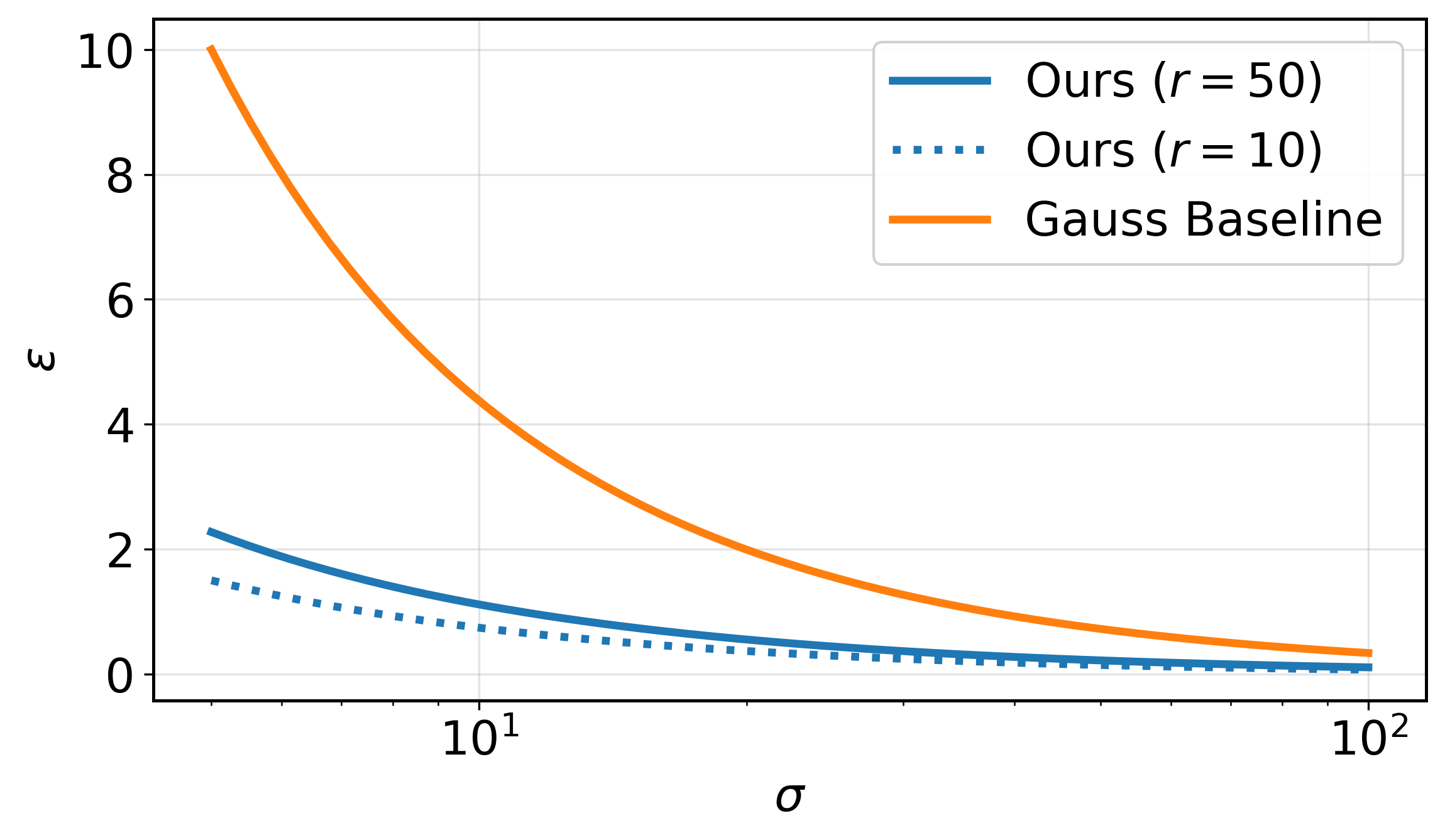}
    \caption{Privacy loss $\varepsilon$ v.s. noise scale $\sigma$ for $\delta = 1e-5$ and $d = 2000$: We compare our bound~(\Cref{thm:matrix-dp-proj-mechanism-small-r}), minimized over $\alpha$ with the Classical Gaussian mechanism~(\Cref{equ:gauss-delta-small-r}).}
    \label{fig:small-r-privacy-amplification}
    \vspace{-5mm}
\end{figure}
On \(\mathcal{G}_\alpha\), the effective squared sensitivity drops from \(\|\Delta V\|_F^2\) to \(\alpha\|\Delta V\|_F^2\), so the Gaussian term improves as \(\alpha\) decreases. Meanwhile \(\Pr(\mathcal{G}_\alpha^c)\) decreases as \(\alpha\) increases. Thus \(\alpha\) trades off a tighter Gaussian term against a rarer good event. This trade-off translates into a concrete improvement once we optimize over \(\alpha\) via simple numerical methods. As shown in \Cref{fig:small-r-privacy-amplification}, the resulting bound can strictly beat the Gaussian baseline in the small-\(r\) regime (with larger gains for smaller \(r\)). Intuitively, a random rank-\(r\) subspace captures only \(\approx r/d\) of the energy of a fixed direction. In the small-\(r\) regime, choosing \(\alpha\) just above this typical level keeps the failure probability small while still substantially tightening the Gaussian term. \Cref{cor:small-r-delta-improvement} in the Appendix makes this precise: if \(r\ll d\) and \(\log s \lesssim r\), an appropriate \(\alpha\) always outperforms the Gaussian baseline.



To relate these bounds to common training procedures, note that the Gaussian baseline above is exactly the privacy accounting for standard DP-SGD. Recalling that by our notation \(V=\nabla_W \cL(W_t)^\top\), privatizing \(V\) directly (without conditioning on \(M\)) amounts to adding Gaussian noise to the full gradient, i.e., updating without any low-rank projection. We use \(\cA_2\) as a proxy for DP-LoRA-FA, but emphasize that it is not identical to the standard implementation: we add Gaussian noise to \(V^\top=\nabla_W \cL(W_t)\) and then apply the random projection, whereas DP-LoRA-FA adds noise after first multiplying by \(A^\top\). We focus on the former variant because it admits a more tractable theoretical analysis. Nonetheless, Appendix~\ref{app:subsec-comparison-dp-lora-small-r} shows that when analysing one training step, the two procedures inject noise of the same order, with high probability. For completeness, we include the standard DP-LoRA-FA implementation in our experiments.

For an empirical comparison, we fine-tune the linear head of a pretrained ResNet-50 on CIFAR-10 and evaluate our method (\ref{eq:M2} with privacy accounting from \Cref{thm:matrix-dp-proj-mechanism-small-r}) against DP-LoRA-FA and DP-SGD under the same privacy budget. Appendix~\ref{app:subsec-comparison-exp-details} provides implementation details (backbone, hyper-parameter grid, and accounting). 

\Cref{fig:comparison} (right) plots test accuracy versus the privacy budget $\varepsilon$ for $\delta = 10^{-4}$. In the small-\(\varepsilon\) regime, our method outperforms DP-SGD because the low-rank projection reduces the noise required to attain a fixed privacy level. DP-LoRA-FA shows the same effect, but only at even smaller \(\varepsilon\). As $\varepsilon$ grows, less noise is required noise for all methods and this noise advantage fades. For large $\varepsilon$ the accuracy gains from reduced noise are outweighed by the optimization bias introduced by the low-rank constraint, so DP-SGD matches or slightly outperforms our approach. 
In~\Cref{fig:comparison} (left), we isolate the effect of the low-rank projection by comparing our method with standard DP-LoRA-FA for varying ranks $r$, at fixed privacy budget ($\varepsilon \in \{0.2, 0.4\}$)\footnote{See results for more $\varepsilon$ in~\Cref{fig:best-acc-vs-r-more-eps}. }. While the two approaches are similar and, with high probability, add effective noise of the same order (see Appendix~\ref{app:subsec-comparison-dp-lora-small-r}), we observe that our Wishart projection consistently performs better, particularly at smaller ranks. This suggests that utility is influenced not only by the noise level, but also by the geometry of the random projection (applied to both the gradient matrix and the noise). Understanding when and why different projection distributions yield better utility is an interesting direction for future work. 

\begin{figure}[t]
    \centering
    \includegraphics[width=0.8\linewidth]{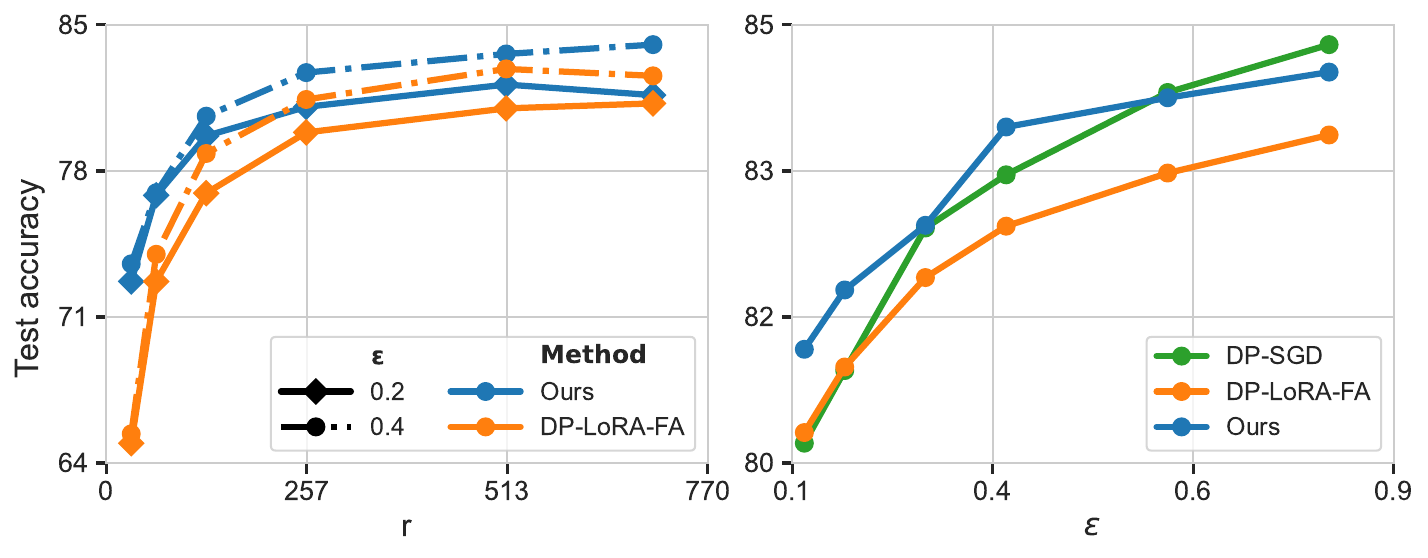}
    \caption{Comparison of DP-SGD, DP-LoRA-FA, and our mechanism~(\ref{eq:M2}) using the privacy accounting of \Cref{thm:matrix-dp-proj-mechanism-small-r}.}
    \label{fig:comparison}
    \vspace{-7mm}
\end{figure}

\section{Discussion and Open Questions}\label{sec:discussion}
\textbf{Related works} Random projections are widely exploited in the privacy literature. Some works, such as \citet{Kenthapadi2013JL,li2023signprojection}, explore the privacy of JL-style projections or random sign flipping. However, these approaches typically do not treat the projection’s randomness as part of the privacy mechanism: they publish the projection matrix and regard its randomness as public information. Other lines of work use projections mainly for dimensionality reduction to improve the privacy–utility trade-off by reducing the dimension dependence in convergence guarantees~\citep{jiang2025balancingrpsgd,kasiviswanathan21a,li2011compressive,or2019oldtechnique}. The intuition is similar to our~\Cref{thm:matrix-dp-proj-mechanism-small-r}, however, those results improve utility in expectation or with high probability over algorithmic randomness, whereas we directly establish a \emph{privacy amplification} at fixed noise and characterize its dependence on the rank \(r\).

Far fewer works exploit projection randomness itself as a privacy resource. ~\citet{lev2025the} show that Gaussian sketching can amplify privacy: releasing a Gaussian sketch with additive noise ($Z^\top V + \Xi$) yields improved privacy guarantees under a ''scale'' assumption on the singular values of the data matrix. Our setting is LoRA-motivated and geometrically different: we project with a Wishart matrix $M = ZZ^\top$ and release~$M V + \Xi$~(\ref{eq:M1}). As $M$ is PSD, the induced geometry differs from Gaussian sketching, leading to a different analysis and assumptions. Nevertheless, the takeaway is the same: intrinsic projection randomness plus modest noise can yield stronger privacy guarantees than additive noise alone.

There is also a line of work on privacy amplification via compression. In particular,~\citet{jin2025signsgd} shows that by compressing the gradient to their signs, amplifies privacy guarantees. Perhaps most related to our work, \citet{Hao2024Flora} interpret LoRA as gradient compression via (re-sampled) low-dimensional random projections, but their motivation is primarily \emph{efficiency}, not privacy. \citet{malekmohammadi2025lowrankadaptationsecretlyimitates} argue via a CLT-based approximation that LoRA can induce DP-SGD-like training dynamics for certain architectures in the limit. In contrast, we provide \emph{non-asymptotic} results for the finite-dimensional updates used in LoRA: we give formal differential privacy guarantees for vector valued updates, a sharp non-privacy result for noise-free matrix updates, and finally show that when adding small amounts of additive noise to the gradients the random projections yield \emph{privacy amplification} beyond noise calibration alone.

\textbf{Discussion} The negative result for matrix-valued updates shows that noise-free matrix-valued LoRA-FA updates are not differentially private. This worst-case statement does not preclude weaker privacy notions. Apart from this, in~\cref{tab:mia-noiseless} we observe perfect success for $r \geq 128$, but slightly lower success when $r$ is extremely small. We attribute this gap primarily to a utility effect: very small-r LoRA does not achieve comparable test accuracy, making loss-based MIA less sensitive. Understanding LoRA’s privacy risk under weaker adversaries, and disentangling genuine privacy gains from ineffective MIA adversaries at small $r$, is an interesting direction for future work.

On the positive side, we show that once additive noise is introduced, the intrinsic randomness of the projection mechanism can be leveraged to amplify privacy. However, for the large $r$ case~(\Cref{subsec:matproj-large-r}) these alignment-based guarantees rely on data-dependent properties that may be difficult to verify or enforce in practice. An interesting direction for future work is to design algorithms or training objectives that promote such alignment, or to exploit it when present to improve privacy--accuracy trade-offs.




\section{Acknowledgement}

JD acknowledges support from the Danish Data Science Academy, which is funded by the Novo Nordisk Foundation (NNF21SA0069429) and VILLUM FONDEN (40516). AS acknowledges the Novo Nordisk Foundation for support via the Startup grant (NNF24OC0087820) and VILLUM FONDEN via the Young Investigator program (72069). The authors would also like to thank Rasmus Pagh, Christian Janos Lebeda, and Vikrant Singhal for insightful discussions.

\bibliographystyle{plainnat}
\bibliography{icml2026}

\appendix
\clearpage
\section*{Appendix}

\section{Mathematical Preliminaries}

We recall a few standard distributional facts used throughout the paper.
\begin{lem}[Basic composition]\label{lem:composition}
If $\cA_1,\dots,\cA_K$ are each $(\varepsilon,\delta)$-DP on the same domain and are run on the same dataset, then the tuple $\br{\cA_1,\dots,\cA_K}$ is $\br{K\epsilon,K\delta}$-DP.
\end{lem}
\begin{defn}[Chi-square distribution.]
A random variable $V$ has a chi-square distribution with $\nu$ degrees of freedom, written
$V \sim \chi^2_\nu$, if it can be represented as
\[
V = \sum_{i=1}^\nu Z_i^2,
\qquad Z_i \overset{\text{i.i.d.}}{\sim} \mathcal{N}(0,1).
\]
\end{defn}

\begin{defn}[Student-$t$ distribution as a ratio, Def.~8.3.1 in \citep{HoggMcKeanCraig2019}]\label{def:student-t}
Let $Z \sim \mathcal{N}(0,1)$ and $V \sim \chi^2_\nu$ be independent. Then the random variable
\[
T_\nu := \frac{Z}{\sqrt{V/\nu}}
\]
follows a (central) Student-$t$ distribution with $\nu$ degrees of freedom, denoted $T_\nu \sim t_\nu$.
Equivalently,
\[
\frac{Z}{\sqrt{V}} \ \overset{d}{=} \ \frac{1}{\sqrt{\nu}}\,T_\nu .
\]
\end{defn}

\begin{lem}[Corollary 7.3.2 in~\cite{Vershynin2018}]\label{lem:gaussian-random-matrices-tails}
        Let $A$ be an $m\times n$ matrix with independent $N(0, 1)$ entries. Then, for $t \geq 0$, we have \[\bP\bs{\norm{A}\geq \sqrt{m} + \sqrt{n} + t} \leq 2e^{
    -ct^2}.\]
    \end{lem}

\begin{lem}\label{lem:singular-value-lower-tail}
    Let $A\in \bR^{m\times r}$ has i.i.d. $N(0, 1)$ entries with $m > r$ and let $\sigma_1(A) \geq ... \geq \sigma_r(A)$ be its singular value, then for any $t \geq 0$, 
    \[\bP\bs{\sigma_r(A)\leq \sqrt{m}-\sqrt{r}-t}\leq e^{-\frac{t^2}{2}}\]
\end{lem}

\begin{lem}[Spectrum of a rank-$r$ Wishart matrix, Thm.~4.6.1 \cite{Vershynin2018}]
\label{lem:wishart-ZZt}
Let $Z\in\mathbb{R}^{d\times r}$ have i.i.d.\ $\mathcal{N}(0,1)$ entries and assume $d\ge r$.
Define the (normalized) Wishart matrix
\[
W \;\coloneqq\; \frac{1}{r}ZZ^\top \in \mathbb{R}^{d\times d}.
\]
Then $\mathrm{rank}(W)=r$, so $W$ has exactly $d-r$ zero eigenvalues. Moreover, for every $t\ge 0$, with probability at least $1-2e^{-t^2/2}$,
\begin{equation}\label{eq:wishart-nonzero-eigs}
\left(\sqrt{\frac{d}{r}}-1-\frac{t}{\sqrt r}\right)^2
\;\le\;
\lambda_{\min}^+(W)
\;\le\;
\lambda_{\max}(W)
\;\le\;
\left(\sqrt{\frac{d}{r}}+1+\frac{t}{\sqrt r}\right)^2,
\end{equation}
where $\lambda_{\min}^+(W)$ denotes the smallest \emph{nonzero} eigenvalue of $W$.
Equivalently, the nonzero spectrum of $W$ lies in the interval above, i.e.
\[
\mathrm{spec}(W)\setminus\{0\}
\subseteq
\left[
\left(\sqrt{\frac{d}{r}}-1-\frac{t}{\sqrt r}\right)^2,\ 
\left(\sqrt{\frac{d}{r}}+1+\frac{t}{\sqrt r}\right)^2
\right]
\]
with probability at least $1-2e^{-t^2/2}$.
\end{lem}

\begin{proof}
By Vershynin's Gaussian singular value bound, for every $t\ge 0$, with probability at least $1-2e^{-t^2/2}$,
\[
\sqrt d-\sqrt r - t \;\le\; s_{\min}(Z) \;\le\; s_{\max}(Z) \;\le\; \sqrt d+\sqrt r + t.
\]
Since the nonzero eigenvalues of $ZZ^\top$ equal $s_i(Z)^2$, the nonzero eigenvalues of
$W=\frac1r ZZ^\top$ equal $\frac{1}{r}s_i(Z)^2$, giving \eqref{eq:wishart-nonzero-eigs}.
Finally, $\mathrm{rank}(ZZ^\top)=\mathrm{rank}(Z)=r$ almost surely, so $W$ has $d-r$ zero eigenvalues.
\end{proof}

\begin{proof}
Use the Kronecker-vec identity
\[
\mathrm{vec}(AX) = (I_k\otimes A)\,\mathrm{vec}(X),
\]
(which is the special case of \(\mathrm{vec}(BXC) = (C^\top\otimes B)\mathrm{vec}(X)\) with \(B=A\) and \(C=I_k\)).

Therefore,
\[
\mathrm{vec}(X)^\top (I_k\otimes A)\,\mathrm{vec}(X)
= \mathrm{vec}(X)^\top \mathrm{vec}(AX).
\]
Now apply the Frobenius inner-product identity
\[
\mathrm{vec}(U)^\top \mathrm{vec}(V) = \operatorname{tr}(U^\top V),
\]
with \(U=X\) and \(V=AX\). This gives
\[
\mathrm{vec}(X)^\top \mathrm{vec}(AX)
= \operatorname{tr}\!\bigl(X^\top (AX)\bigr)
= \operatorname{tr}(X^\top A X).
\]
\end{proof}

\begin{lem}[Tail bound for random capture fraction]
\label{lem:beta-tail}
Let $r\le d/2$ and let
\[
B \sim \mathrm{Beta}\!\left(\frac r2,\frac{d-r}{2}\right).
\]
Fix any $\eta\in(0,1)$ and set $\alpha \;\coloneqq\; \frac{r(1+\eta)}{d}\;\in(0,1)$.
Then
\[
\Pr(B>\alpha)\;\le\;2\exp\!\left(-\frac{\eta^2 r}{72}\right)
\]
\end{lem}

\begin{proof}
Let $X\sim\chi^2_r$ and $Y\sim\chi^2_{d-r}$ be independent. It is standard that
\[
B \;\overset{d}{=}\; \frac{X}{X+Y}.
\]
We claim that the event $\{B>\alpha\}$ is contained in the union of two simpler deviations:
\begin{equation}\label{eq:union}
\{B>\alpha\}
\;\subseteq\;
\Bigl\{X>(1+\eta/3)r\Bigr\}
\;\cup\;
\Bigl\{Y<(1-\eta/3)(d-r)\Bigr\}.
\end{equation}
Indeed, suppose that both complementary events hold, i.e.
\[
X\le (1+\eta/3)r
\qquad\text{and}\qquad
Y\ge (1-\eta/3)(d-r).
\]
Then
\[
B=\frac{X}{X+Y}
\;\le\;
\frac{(1+\eta/3)r}{(1+\eta/3)r + (1-\eta/3)(d-r)} = \frac{(1+\eta/3)r}{d - \eta/3(d-2r) }.
\]
We claim that
\[
\frac{(1+\eta/3)\,r}{\,d-\frac{\eta}{3}(d-2r)\,}\ \le\ \frac{(1+\eta)\,r}{d} = \alpha.
\]
Since $r>0$ and $d-\frac{\eta}{3}(d-2r)>0$, this is equivalent to
\[
(1+\eta)\!\left(d-\frac{\eta}{3}(d-2r)\right) - \left(1+\frac{\eta}{3}\right)d \ \ge\ 0.
\]
Expanding, we obtain
\[
(1+\eta)\!\left(d-\frac{\eta}{3}(d-2r)\right) - \left(1+\frac{\eta}{3}\right)d
= \frac{\eta(1-\eta)}{3}\,d \;+\; \frac{2\eta(1+\eta)}{3}\,r.
\]
This is nonnegative since $r,d>0$ and $\eta\in(0,1)$, which proves \eqref{eq:union}.

Next, we apply standard chi-square concentration: for $Z\sim\chi^2_k$ and any $t\in(0,1)$,
\begin{equation}\label{eq:chisq-tail}
\Pr\!\big(Z>(1+t)k\big)\le \exp\!\left(-\frac{t^2 k}{8}\right),
\qquad
\Pr\!\big(Z<(1-t)k\big)\le \exp\!\left(-\frac{t^2 k}{8}\right).
\end{equation}
Using \eqref{eq:union} with $t=\eta/2$ and a union bound gives
\[
\Pr(B>\alpha)
\;\le\;
\Pr\!\big(X>(1+\eta/3)r\big)
+
\Pr\!\big(Y<(1-\eta/3)(d-r)\big).
\]
Applying \eqref{eq:chisq-tail} yields
\[
\Pr\!\big(X>(1+\eta/3)r\big)
\le
\exp\!\left(-\frac{(\eta/3)^2 r}{8}\right)
=
\exp\!\left(-\frac{\eta^2 r}{72}\right),
\]
and similarly
\[
\Pr\!\big(Y<(1-\eta/2)(d-r)\big)
\le
\exp\!\left(-\frac{(\eta/3)^2(d-r)}{8}\right)
\le
\exp\!\left(-\frac{\eta^2 r}{72}\right),
\]
where the last inequality uses $d-r\ge r$ since $r\le d/2$. Therefore,
\[
\Pr(B>\alpha)
\le
2\exp\!\left(-\frac{\eta^2 r}{72}\right)
\]
where the final step simply loosens the constant for a cleaner expression. 
\end{proof}

\begin{defn}[Orthogonal group]\label{def:orthogonal-group}
The orthogonal group is
\[
O(d)\;\coloneqq\;\{U\in\R^{d\times d}:\; U^\top U = I_d\}.
\]
\end{defn}

\begin{defn}[Grassmannian]\label{def:grassmannian}
The Grassmannian \(\mathrm{Gr}(d,r)\) is the set of all \(r\)-dimensional linear subspaces of \(\R^d\):
\[
\mathrm{Gr}(d,r)\;\coloneqq\;\{S\subseteq \R^d:\; S \text{ is a linear subspace and } \dim(S)=r\}.
\]
\end{defn}

\begin{defn}[Stiefel manifold]\label{def:stiefel}
The (real) Stiefel manifold \(V_{d,r}\) is the set of all \(d\times r\) matrices with orthonormal columns:
\[
V_{d,r}\;\coloneqq\;\{Q\in\R^{d\times r}:\; Q^\top Q = I_r\}.
\]
\end{defn}

\begin{defn}[Haar-uniformity on the Stiefel manifold]\label{def:haar-stiefel}
A random element \(Q\in V_{d,r}\) is \emph{Haar-uniform} on the Stiefel manifold if, for every \(U\in O(d)\),
\[
UQ \;\overset{d}{=}\; Q.
\]
Equivalently, the law of \(Q\) is the unique probability measure on \(V_{d,r}\) that is invariant under the left action \(Q\mapsto UQ\).
\end{defn}

\begin{defn}[Uniformity on the Grassmannian]\label{def:uniform-grassmannian}
A random subspace \(S\in \mathrm{Gr}(d,r)\) is \emph{uniform on the Grassmannian} if, for every \(U\in O(d)\),
\[
US \;\overset{d}{=}\; S,
\]
where \(US\coloneqq\{Ux:\;x\in S\}\).
Equivalently, the law of \(S\) is the unique probability measure on \(\mathrm{Gr}(d,r)\) that is invariant under the action \(S\mapsto US\).
\end{defn}

\section{Privacy Analysis of Vector Projection Mechanism}

\vecprivacynew*

\begin{proof}
Fix neighboring datasets $S \sim_H S'$ in $\mathcal D$ and set
$v := (f(S))^\top, v' := (f(S'))^\top \in $, with $\|v\|_2=\|v'\|_2=1$ and
$\langle v, v'\rangle =: \rho \in (0,1]$.
Let $Z\in\bR^{d\times r}$ have i.i.d. columns $z_k\sim \bN(0, \frac{1}{r}I_d)$,
and define $M := ZZ^\top$.
The mechanism outputs $Y := f(S) M = v^\top M$. We will show that the mechanism $v\mapsto Mv$ is $(\varepsilon_\rho, \delta_\rho)$-DP. By post-processing property of DP~(\Cref{lem:dp-postprocessing}), \Cref{defn:projection-mech} ($V^\top \mapsto V^\top M$, equivalently,~\Cref{defn:projection-mech}) is also DP.

Let $P$ and $Q$ denote the laws of $Y$ under inputs $v$ and $v'$ respectively (i.e.$Y = Mv$ and $Y = Mv'$),
with densities $p$ and $q$.

By PDF of $Mv$~(\Cref{lem:pdf-mv}), $p(y)$ has the form
    \[
         p(y) = C_{r, d, \sigma}(v^\top y)^{\frac{r-d-1}{2}} \exp\br{-\frac{\norm{y}^2}{2 \sigma^2 v^\top y}}
    \]
    on the half-space $\{Y: v^\top y> 0\}$. Analogously, $q(y)$ has the same form with $v$ replaced by $v'$, and is supported on $\{ y: (v')^\top y > 0 \}$. Define the support of $q$ by \[\cZ_q := \{ y: q(y) > 0 \} = \{ y: (v')^\top y > 0 \}.\] On $\cZ_q$, define the privacy loss random variable for $y\sim q$, \[L(y) = \ln \frac{p(y)}{q(y)} \]
   Now for any measurable set $\cY\subset \bR^d$ we have , 
    \begin{equation}\label{eq:vec-privacy-failure-derivation-1}
        \begin{aligned}
        p(\cY) &= p(\cY\cap \cZ_q) + p(\cY\cap \cZ_q^c) \leq p(\cZ_q^c) + \int_{\cY\cap \cZ_q} p(y) dy \\
        \end{aligned}
    \end{equation}
    On $\cZ_q$ we have $p(y) = e^{L(y)}q(y)$, so

    \begin{equation}\label{eq:vec-privacy-failure-derivation-2}
        \begin{aligned}
        \int_{\cY\cap \cZ_q} p(y) dy &= \int_{\cY\cap \cZ_q} e^{L(y)}q(y) dy\\
        &\overset{(a)}{\leq} e^\varepsilon q(\cY \cap \cZ_q) + \int_{\cY\cap \cZ_q} e^{L(y)} \mathbf{1}\{L(y) \geq \varepsilon\}q(y)dy\\
        &\overset{(b)}{\leq} e^\varepsilon q(\cY) + \int_{\cY\cap \cZ_q} e^{L(y)} \mathbf{1}\{L(y) \geq \varepsilon\}q(y)dy.
        \end{aligned}
    \end{equation}
    where step (a) follows by splitting the integrand into whether $L(y)< \varepsilon $ or $L(y) \geq \varepsilon$, and step (b) is due to $q(\cY \cap \cZ_q) \leq q(\cY)$. 

    Substituting~\Cref{eq:vec-privacy-failure-derivation-2} into~\Cref{eq:vec-privacy-failure-derivation-1}, we get 
    \begin{equation}\label{eq:vec-privacy-failure-derivation}
    \begin{aligned}
        p(\cY) &\leq e^\varepsilon q(\cY) + p(\cZ_q^c) + \int_{\cZ_q} e^{L(y)} \mathbf{1}\{L(y) \geq \varepsilon\}q(y)dy \\
        &= e^\varepsilon q(\cY) + p(\cZ_q^c) + \int_{\cZ_q} \mathbf{1}\{L(y) \geq \varepsilon\}p(y)dy \\
        &\leq e^\varepsilon q(\cY) + p(\cZ_q^c) + \mathbb{P}_{y \sim p}(L(Y) \geq \varepsilon, y \in \cZ_q)
    \end{aligned}
    \end{equation}
    where the last inequality is due to $\cY\cap \cZ_q \subset \cZ_q$.

    Next, we upper bound the second term $p(\cZ_q^c)$. We have that
    \[p(\cZ_q^c) = \bP_{y \sim p}( (v')^\top y \leq 0)  = \bP_{M}((v')^\top M v \leq 0).\]
    For $k\in [r]$ $z_k$ be the kth column of $Z$, $z_k \sim \cN(0, 1/rI_d)$. Let $v'=\rho v + \sqrt{1-\rho^2} w$ where $w\perp v, \norm{w}_2 = 1$. Let the unit vector 
    \[
    w :=
    \begin{cases}
    \dfrac{v' - \rho\, v}{\sqrt{1-\rho^{2}}}, & \text{if } |\rho|<1,\\[6pt]
    \text{any unit vector in } v^\perp, & \text{if } |\rho|=1.
    \end{cases}
    \]
    
    Let $g_k := \sqrt{r}v^\top z_k$ and $h_k := \sqrt{r} w^\top z_k$.
    Then $g_k,h_k \stackrel{i.i.d.}{\sim}\mathcal N(0,1)$ and are independent across $k$,
    and
    \[
    (v')^\top Y = (v')^\top Mv
    = \sum_{k=1}^r (v'^\top z_k)(v^\top z_k)
    = \frac{1}{r}\sum_{k=1}^r (\rho g_k + \sqrt{1-\rho^2}\,h_k)g_k.
    \]
    Define
    \[
    X := \sum_{k=1}^r g_k^2 \sim \chi_r^2,
    \qquad
    S := \sum_{k=1}^r g_k h_k.
    \]
    Conditioned on $(g_1,\dots,g_r)$, $S$ is a linear combination of independent $h_k$'s,
    so
    \[
    S\mid (g_1,\dots,g_r) \sim \mathcal N(0, X).
    \]
    By rotational symmetry this implies $S\mid X=x \sim \mathcal N(0,x)$, 
    Hence, \[(v')^\top Y |(g_1, \dots, g_r) \sim \cN\br{\frac{\rho}{r}\sum_{k = 1}^r g_k^2, \frac{1-\rho^2}{r^2}\text{Var}(S|g_1, \dots g_r)} \overset{d}{=}\cN\br{ \rho X, \frac{1-\rho^2}{r^2}X}
    \]
    and similarly,
    \[
    (v')^\top Y \mid X=x \sim \mathcal N\!\Big(\frac{\rho}{r}x,\; \frac{1-\rho^2}{r^2}x\Big),
    \]
    and therefore
    \[
    P\big((v')^\top Y\le 0\mid X=x\big)
    = \Phi\!\Big(-\frac{\rho\sqrt x}{\sqrt{1-\rho^2}}\Big).
    \]
    Taking expectation over $X\sim\chi_r^2$ yields
    \begin{equation}\label{eq:delta-support}
    P\big((v')^\top Y\le 0\big)
    = \mathbb E_{X\sim\chi_r^2}\!\left[\Phi\!\Big(-\frac{\rho\sqrt X}{\sqrt{1-\rho^2}}\Big)\right].
    \end{equation}

    So we are left to choose an $\varepsilon_\rho$ in a reasonable range so that $\mathbb{P}_{y \sim p}(L(y) \geq \varepsilon_\rho) \leq \delta'$ lies in a reasonable range:
    
    Assume $d > r$, for $y \in \cZ_q$ the privacy loss random variable satisfies
    \begin{equation}
        \begin{aligned}
            L(y) &= \log \frac{P(y)}{Q(y)} = \frac{d - r + 1}{2}\log\br{\frac{(v')^\top y}{v^\top y}}+ \frac{r\norm{y}^2}{2v^\top y}\br{\frac{v^\top y}{(v')^\top y}-1} \\
            &= \frac{d-r+1}{2}\log A + \frac{B}{2}\br{\frac{1}{A}-1}
        \end{aligned}
    \end{equation}
    where $A\coloneqq \frac{(v')^\top y}{v^\top y}$ and $B\coloneqq \frac{r\norm{y}^2}{v^\top y}$, with $y\sim p$. $A$ quantifies relative alignment of $y$ with $v'$ vs $v$, and $B$ quantifies how large $y$ is compared to its $v$-projection. This means $L(y)$ will blow up when either $B$ is very large or $A$ is close to 0.By~\Cref{lem:control-A-and-B} we have
    \[
    A \ \overset{d}{=} \ \rho + \sqrt{\frac{1-\rho^2}{r}}\,T_r,
    \qquad
    B \ \overset{d}{=} \ \chi_{d+r-1}^2.
    \]
    so for $\delta' > 0$, e, $t_{r}(\cdot)$ denote the quantile function of a student t distribution with degree of freedom $r$. Let \[a_- = \rho - \sqrt{\frac{1 - \rho^2}{r}}t_{r}\br{ 1-\frac{\delta'}{3}}, \quad a_+ = \rho + \sqrt{\frac{1 - \rho^2}{r}}t_{r}\br{1-\frac{\delta'}{3}}\]

    Let $\kappa_{d + r - 1}(\cdot)$ denote the quantile function of $\chi_{d + r-1}^2$. Let $b = \kappa_{d + r - 1}\br{ 1-\frac{\delta'}{3}}$. 
    Define the good set as $A\in (a_-, a_+)$, $B \leq b$. As the good set implies an upper bound on $L(y)$, \[\bP_{y\sim p}\bs{L(y) \leq \frac{d-r + 1}{2}\ln a_+ + \frac{b}{2}\br{\frac{1}{a_-}-1}}\geq \bP(\text{good set})\]
    By the definition of quantile function, 
    \begin{align*}
        \bP(\text{good set}) &= \bP(A\in (a_-, a_+) \cap B \leq b) = 1- \bP(A \notin (a_-, a_+) \cup B > b) \\&\geq 1 -  \bP(A \notin (a_-, a_+)) - \bP(B > b) = 1-\delta'.
    \end{align*}
     Thus, let $\varepsilon(a_-, a_+, b) \coloneqq \frac{d-r + 1}{2}\ln \frac{1}{a_+} + \frac{b}{2}\br{\frac{1}{a_-}-1}$, \[\bP_{y\sim p}\bs{L(y) > \varepsilon(a_-, a_+, b)}\leq 1-\bP(\text{good set}) \leq \delta'\]

    \end{proof}

    \begin{lem}[Distributional representation of $A$ and $B$]\label{lem:control-A-and-B}
Let $Z \in \mathbb{R}^{d \times r}$ have i.i.d.\ $N(0,1)$ entries and set
\[
M := \frac{1}{r}ZZ^\top .
\]
Fix unit vectors $v, v' \in \mathbb{R}^d$ with inner product
\[
\rho := \langle v, v' \rangle \in [-1,1],
\]
and let $y := Mv$. Define
\[
A := \frac{(v')^\top y}{v^\top y},
\qquad
B := \frac{r\|y\|^2}{v^\top y}.
\]
Then there exist independent random variables
\[
K_1 \sim \chi_r^2,
\qquad
K_2 \sim N(0,1),
\qquad
K_3 \sim \chi_{d-2}^2,
\]
such that, jointly,
\[
(A,B) \ \overset{d}{=} \ \left(\rho + \sqrt{1-\rho^2}\,\frac{K_2}{\sqrt{K_1}},\ \ K_1 + K_2^2 + K_3\right).
\]
In particular, if
\[
T_r := \frac{K_2}{\sqrt{K_1/r}} \sim t_r
\]
is Student-$t$ with $r$ degrees of freedom, then
\[
A \ \overset{d}{=} \ \rho + \sqrt{\frac{1-\rho^2}{r}}\,T_r,
\qquad
B \ \overset{d}{=} \ \chi_{d+r-1}^2.
\]
\end{lem}

\begin{proof}
    Recall that $M = ZZ^\top/r$ where $Z$ is $d$ by $r$ with $i.i.d$ Gaussian entries $N(0, 1)$. 
    Let $Q \in \bR^{d\times d}$ be an orthonormal matrix such that $Qv = e_1, Q v' = \rho e_1 + \sqrt{1-\rho^2} e_2$. Then, for $y\sim q$, 
    \begin{align*}
        A \coloneqq \frac{(v')^{\top} y}{v^{\top} y} &\overset{d}{=} \frac{(v')^\top  ZZ^\top v}{ v^\top ZZ^\top v}\overset{d}{=} \frac{(v')^\top  QZZ^\top Q^\top v}{ v^\top Q ZZ^\top Q^\top v} \\
        &= \frac{(\rho e_1 + \sqrt{1-\rho^2}e_2)^\top ZZ^\top e_1}{e_1^\top ZZ^\top e_1}\\
       &= \rho + \sqrt{1-\rho^2} \frac{e_2^\top ZZ^\top e_1}{e_1^\top ZZ^\top e_1} \overset{d}{=} \rho + \sqrt{1 - \rho^2} \frac{g_2^\top g_1}{g_1^\top g_1} 
    \end{align*}
    where $g_1 = Ze_1, g_2 = Z e_2 \sim \cN(0, I_r)$ and $g_1$ is independent of $g_2$. 
    \[
        \frac{g_2^\top g_1}{g_1^\top g_1} = \frac{g_2^\top \frac{g_1}{\norm{g_1}}}{\norm{g_1}} \overset{d}{=} \frac{K_2}{\sqrt{K_1}}\]
        for $K_2\sim N(0, 1)$, $K_1\sim \chi_r^2$ and $K_2 $ independent of $K_1$. 
    As $g_1\sim \cN(0, I_r)$, $\frac{g_1}{\norm{g_1}}$ is independent of $\norm{g_1}$. Thus, $g_2^\top g_1/\norm{g_1}$ is independent of $\norm{g_1}$ and the last equality follows. So by~\Cref{def:student-t} we can write $A \overset{d}{=} \rho + \sqrt{\frac{1-\rho^2}{r}}T_r$ where $T_r$ is a random variable following the student-t distribution with degree of freedom $r$.

    \begin{equation*}
        \begin{aligned}
            B &\coloneqq \frac{r\norm{y}^2}{2v^\top y} = \frac{e_1 ZZ^\top ZZ^\top e_1}{e_1 ZZ^\top e_1} \\
            &= \frac{g_1 ^\top Z^\top Zg_1}{g_1^\top g_1} = \frac{\sum_{i = 1}^d (g_i^\top g_1)^2}{g_1^\top g_1} \\
            &= g_1^\top g_1 + \sum_{i = 2}^d \br{g_i^\top \frac{g_1}{\norm{g_1}}}^2 \overset{d}{=} \chi_{r + d-1}^2
        \end{aligned}
    \end{equation*}
    where the last inequality follows as conditional on $g_1$ we have that for $ u = g_1/\|g_1\|$, $g_i^\top u \sim \mathcal{N}(0,1)$ and therefore conditional on $g_1$ for each $i \geq 2$ we have $\sum_{i=2}^d(g_i^\top u)^2 \sim \chi_{d-1}^2$. And lastly $\sum_{i=2}^d(g_i^\top u)^2 \sim \chi_{d-1}^2$ is independent of $g_1^\top g_1 = \|g_1||_2^2$, so we can remove the conditioning. 
\end{proof}

\begin{lem}[PDF of $Mv$]\label{lem:pdf-mv}
Let $z_1, ..., z_r$ be i.i.d. $\mathcal{N}(0, \sigma^2 I_d)$ where $d \geq r$, $M = \sum_{i = 1}^r z_i z_i^T$, then for $v\in \bR^{d}$ with $\norm{v} = 1$ and $y \in \bR^d$ such that $v^\top y > 0$, 
\[\bP(Mv = y) =  C_{r, d, \sigma}(v^\top y)^{\frac{r-d-1}{2}}
\exp\br{-\frac{\norm{y}^2}{2 \sigma^2 v^\top y}}\]
where $C_{r, d, \sigma} = \frac{1}{2^{r/2}\Gamma(r/2)\sigma^{d-r- 1}(2\pi)^{(d-1)/2}}$
\end{lem}
\begin{proof}
    For \[Y = Mv = \sum_{i = 1}^r z_i z_i^\top v,\]
    let \[a_i = z_i^\top v, \quad u_i = (I - vv^\top) z_i =: P_\perp z_i. \]

    Therefore, we can write $z_i = vv^\top z_i + (I - vv^\top) z_i = v a_i + u_i$, and $Y$ as 
    \[
        Y = \sum_{i = 1}^r z_i z_i^\top v = \sum_{i = 1}^r (v a_i + u_i) (a_i v^\top + u_i^\top) v = \sum_{i = 1}^r v a_i^2 + u_i a_i,
    \]
    where $a_i \sim \cN(0, \sigma^2), u_i \sim \cN(0, \sigma^2 P_\perp )$. Further $a_i$ and $u_i$ are independent as 
    \[
        Cov(a_i, u_i) = \mathbb{E}[a_i \cdot u_i] = \mathbb{E}[ z_i^T v P_{\bot}z_i] = \mathbb{E}[P_{\bot}z_i z_i^T v ] = P_{\bot} \mathbb{E}[z_i z_i^T]v = \sigma^2 P_{\bot} v = 0.
    \] 
    Let $S = \sum_{i = 1}^r a_i^2$, then 
    \[
        S \sim \sigma^2 \chi_r^2, \quad Y|S=s \overset{d}{=}N\br{sv, \sigma^2 sP_\perp}.
    \]
    Define $U\in \cR^{d\times (d-1)}$ as $[u_1 \cdots u_{d-1}] \in \mathbb{R}^{d \times (d-1)}$, where $\{ u_1, \dots, u_{d-1}\}$ is an orthonormal basis of $v^{\bot}$ then \[U^\top U = I_{d-1}, \quad UU^\top = I - vv^\top = P_\perp.\]
    Let $Y_\perp = U^\top Y$, then 
    \[S \sim \sigma^2 \chi_r^2, \quad Y_\perp|S=s \overset{d}{=}N\br{0,  \sigma^2 sU^\top P_\perp U} \overset{d}{=} N\br{0, \sigma^2 sI_{d-1}}.\]
    We note the last inequality follows from substituting $P_\perp$ and noticing \[U^\top v = (U^\top U) U^\top v = U^\top (UU^\top v) = U^\top 0 = 0.\] 
    Therefore, for $y\in \{y\in \bR^d: y^\top v > 0\}$, 
    \begin{equation}
        \begin{aligned}
        \bP\bs{S = s, Y_\perp = y_\perp} &= \bP\bs{S = s}\bP\bs{Y_\perp = y_\perp |S =  s}\\
        \end{aligned}
    \end{equation}
    \begin{equation}
        f_{S\sigma^2}(s) = f_{S}\br{\frac{s}{\sigma^2}} = \frac{s^{r/2-1}e^{-\frac{s}{2\sigma^2}}}{2^{r/2}\Gamma(r/2)\sigma^{r-2}}
    \end{equation}
    \begin{equation}
        f_{Y_\perp|S}(y_\perp|S = s) = \br{2\pi}^{-\frac{d-1}{2}}\br{\sigma^2s}^{-\frac{d-1}{2}}\exp\br{-\frac{\norm{y_\perp}^2}{2\sigma^2s}}
    \end{equation}
    \begin{equation}
        f_{S, Y_\perp}(s, y_\perp) = C_{r, d, \sigma} s^{\frac{r-d + 1}{2}}e^{-\frac{s^2 + \norm{y_\perp}^2}{2\sigma^2s}},\quad C_{r, d, \sigma} = \frac{1}{2^{r/2}\Gamma(r/2)\sigma^{d-r- 1}(2\pi)^{(d-1)/2}}
    \end{equation}
    
    As we can write $\begin{pmatrix}
        S\\Y_\perp 
    \end{pmatrix} = \begin{pmatrix}
        v^\top \\U^\top
    \end{pmatrix}Y$ (using point wise multiplication), let $Q = \begin{pmatrix}
        v^\top \\U^\top
    \end{pmatrix}$. One can easily verify that $Q^\top = Q^{-1}$ and 
    \begin{equation}\label{eq:y-equation}
        Y = Q^\top \begin{pmatrix}
            S\\Y_\perp
        \end{pmatrix}.
    \end{equation}
    By changing the variables from $(S, Y_\perp)$ to $Y$ with~\Cref{eq:y-equation}, we get the probablity density function for $Y$ when $y^\top v \geq 0$ ($s > 0)$, i.e. 
    \[\bP(Y = y) =  C_{r, d, \sigma}(v^\top y)^{\frac{r-d + 1}{2}}
    \exp\!\left( - \frac{r \norm{U^\top y}^2}{2v^\top y} \right)\]
    \begin{equation}
        \exp\br{-\frac{r}{2}\br{s + \frac{\norm{U^\top y}^2}{v^\top y}}} = \exp\br{-\frac{r}{2v^\top y}\br{y^\top (vv^\top + UU^\top )y }} = \exp\br{-\frac{r\norm{y}^2}{2v^\top y}}
    \end{equation}
    So we get 
    \begin{equation}
        f_Y(y) = C_{r, d, \sigma}(v^\top y)^{\frac{r-d + 1}{2}}\exp\br{-\frac{\norm{y}}{2\sigma^2 v^\top y}}
    \end{equation}
\end{proof}

\subsection{Vector Privacy Amplification and Applications}\label{app:vecprivamplification-and-applications}
\paragraph{Privacy amplification by increasing effective alignment} The privacy guarantees for random projection in~\Cref{thm:vec-alt} can be strengthened by introducing a simple \emph{pre-processing strategy}. We add uniform noise from a $d$-dimensional ball of radius $\gamma/2$ to $f(S)$ before applying the projection. Specifically, 
\[
    M \br{ f(S) + \frac{\gamma z}{\norm{z}}}, \quad z \sim N(0, \mathbf{I}_d)
\]
This improves the effective alignment, especially when the original alignment $\rho$ is small (or even negative) and in high-dimensional settings. 
 

 \begin{restatable}{lem}{Amplification}\label{lem:amplification}
     Let $v, v'\in \bR^d$ be two unit vectors with with $\cos\angle\br{v, v'} = v^\top v' \geq \rho$, $z\in \cN(0, \mathbf{I}_d)$, $\delta> 0$ and $\gamma > \frac{1-\rho}{1 + \rho}\sqrt{\frac{2}{d}\log \frac{8}{\delta}}$, then with probability at least $1-\delta$, we have \[\cos\br{\angle\br{v + \frac{\gamma z}{\norm{z}_2}, v' + \frac{\gamma z}{\norm{z}_2}}}\geq \rho + s > \rho,\]
    where $s = \frac{(1-\rho)\gamma^2 - 4\gamma\sqrt{\frac{2}{d}\ln\frac{8}{\delta}}}{1 + \gamma^2 + 2\gamma\sqrt{\frac{2}{d}\ln\frac{8}{\delta}}}$.
 \end{restatable}

We observe that achieving a fixed target improvement \(s\) in alignment requires choosing a larger
\(\gamma\) and adding more noise when the minimum alignment \(\rho\) is large (i.e., when the original vectors are already well aligned).

\subsection{Applications}\label{sec:proj-vec-appl}
In this section, we highlight three potential applications of the projection mechanism for the case. In~\Cref{sec:proj-mechanism-matrix}, we highlight our main application, differentially privarte LoRA .

\noindent\textbf{Projected gradient descent (RP\textendash GD).} Analogous to DP\textendash GD, which privatises gradients by additive noise, we privatise the \emph{average gradient direction} via the projection mechanism and then take a descent step with the projected output. Concretely, sample \(M\sim \mathsf{W}_d\!\br{\sigma^2 I_d,r}\) once, and at each iteration update
\[
w_{t+1} \;=\; w_t \;-\; \eta\, M \nabla \cL\!\br{w_t}.
\] This \emph{Randomly Projected Gradient Descent (RP\textendash GD)} algorithm retains directional information (which is what drives progress for many optimisers) while providing guaranteeing DP. 

\noindent\textbf{Private Retrival} Another possible application is to publish \emph{private embeddings for retrieval tasks}. Given a unit\mbox{-}normalised average embedding \(v\), sample \(M\sim\mathsf{W}_d\br{\sigma^2 I_d,r}\) and release the \(y=Mv\). The retrieval system maintains its catalogue \(\{u_j\}\subset\reals^d\) unchanged and ranks by standard dot products \(\ip{u_j}{y}=\ip{u_j}{Mv}\). Since \(\bE[M]=r\sigma^2 I_d\)~(unlike projections like the JL transformation) and \(\|M-r\sigma^2 I_d\|\) concentrates for moderate \(r\), these scores approximate a constant multiple of \(\langle u_j, v\rangle\), preserving top-\(k\) ordering up to a small distortion that vanishes as \(r\) grows. This is useful for various modern retrival applications, where the embedding \(v\) is computed as an average of multiple embeddings. The same pattern applies to \emph{releasing class/cohort prototypes}: compute the cohort mean, normalise and release \(y=Mv\). In short, any application where the original embedding is an average embedding and final utility is measured with respect to cosine angle is a good fit for the projection mechanism.

\section{Matrix Projection Mechanism is not private}
In this section, we prove~\Cref{lem:image-overlap-measure-zero} that directly implies~\Cref{prop:neg-res}.

\begin{lem}[Almost-sure separation of images under random $M$]\label{lem:image-overlap-measure-zero}
Let $V,V'\in\mathbb R^{d\times m}$ with $\Delta V:=V-V'\neq 0$, and let
$M=ZZ^\top$ where $Z\in\mathbb R^{d\times r}$ has i.i.d.\ $\mathcal N(0,1)$ entries.
Then
\[
\mathbb P\big(MV=MV'\big)\;=\;\mathbb P\big(M\Delta V=0\big)\;=\;0.
\]
In particular, the two random images $\{MV: M\}$ and $\{MV': M\}$ intersect only on
a $\mathbb P$-null set (with randomness over $M$).
\end{lem}

\begin{proof}
We have $MV=MV'$ iff $M\Delta V=0$. Since $M=ZZ^\top$ is positive semidefinite,
for any vector $x$,
\[
ZZ^\top x=0 \quad\Longleftrightarrow\quad x^\top ZZ^\top x=\|Z^\top x\|_2^2=0
\quad\Longleftrightarrow\quad Z^\top x=0.
\]
Applying this columnwise shows $M\Delta V=0 \iff Z^\top \Delta V=0$.

Let $s=\mathrm{rank}(\Delta V)\ge 1$ and write $\Delta V=UB$ where
$U\in\mathbb R^{d\times s}$ has orthonormal columns and $B\in\mathbb R^{s\times m}$
has full row rank. Then
\[
Z^\top\Delta V=0 \ \Longrightarrow\ (Z^\top U)B=0 \ \Longrightarrow\ Z^\top U=0,
\]
since $B$ has full row rank. But $Z^\top U\in\mathbb R^{r\times s}$ has i.i.d.\
$\mathcal N(0,1)$ entries (orthogonal invariance), hence
$\mathbb P(Z^\top U=0)=0$. Therefore $\mathbb P(M\Delta V=0)=0$.
\end{proof}

\subsection{Privacy implication for standard LoRA}\label{app:negative-result-lora}

Here, we detail how~\Cref{subsec:matrix-negative-result} rules out intrinsic privacy for full LoRA. In the LoRA-FA setting, the noise-free effective weight update forms a matrix-valued Wishart projection:
\begin{equation}
    W_{t+1} - W_t = -\eta \nabla_W L(W_t) (A^\top A)
\end{equation}
where $A^\top A$ represents a rank-$r$ Wishart draw. We observe that standard LoRA exhibits identical behavior at initialization. Under the standard initialization $B_0=0$, the first step of full LoRA (even with a trainable $A$) yields the update:
\begin{equation}
    W_1 - W_0 = -\eta \nabla_W L(W_0) A_0^\top A_0
\end{equation}
Consequently, we can post-process the LoRA output $W_1$ by $(W_0 - W_1) \eta$ to get the output of projection mechanism. By the post-processing property of DP, if the noise-free LoRA mechanism were $(\varepsilon, \delta)$-DP, then the projection mechanism would also be $(\varepsilon, \delta)$-DP. However, this contradicts~\Cref{subsec:matrix-negative-result}, which establishes that such projections are not private. Therefore, we conclude that LoRA is not private without additive noise. 
\section{Privacy Analysis of Matrix Projection Mechanism}
\subsection{Proofs large r regime}\label{app:matrix-large-r}
\paragraph{Notation and setup.}
In this section we denote neighboring datasets as $V\sim_A^{(j)}V'$ and let
$S=\mathrm{span}(V_{-j})$ with $\dim(S)=s$.
Let $U\in\mathbb{R}^{d\times s}$ have orthonormal columns spanning $S$, and let
$U_\perp\in\mathbb{R}^{d\times(d-s)}$ have orthonormal columns spanning $S^\perp$.

Let $Z\in\mathbb{R}^{d\times r}$ have i.i.d.\ $\mathcal{N}(0,\sigma_M^2)$ entries.
We define the Gaussian blocks
\[
G := U^\top Z \in\mathbb{R}^{s\times r},\qquad
W := U_\perp^\top Z \in\mathbb{R}^{(d-s)\times r}.
\]
By rotational invariance of the Gaussian and orthogonality of $[U\,\,U_\perp]$,
the matrices $G$ and $W$ are independent and have i.i.d.\ $\mathcal{N}(0,\sigma_M^2)$ entries, and
\[
Z = UG + U_\perp W.
\]

Further, we define $H:=\mathrm{rowspan}(G)\subseteq\mathbb{R}^r$, let $P_H$ and $P_H^\perp$
be the orthogonal projectors onto $H$ and $H^\perp$, and set
\[
Z_\parallel := ZP_H,\qquad Z_\perp := ZP_H^\perp,\qquad
M_\parallel := Z_\parallel Z_\parallel^\top,\qquad M_\perp := Z_\perp Z_\perp^\top,
\]
with $p:=\dim(H)=\mathrm{rank}(G)\le \min\{s,r\}$.

\begin{lem}[Exact orthogonal split]
\label{lem:split}
We have $M=M_\parallel+M_\perp$ and $Z_\parallel Z_\perp^\top=0$.
Moreover,
\[
U^\top Z_\perp = 0 \quad\text{and hence}\quad \mathrm{range}(M_\perp)\subseteq S^\perp.
\]
In particular, $M_\perp a=0$ for all $a\in S$.
    
\end{lem}

\begin{proof}
Recall that $P_H$ and $P_H^\perp$ are orthogonal projectors onto $H$ and $H^\perp$, hence
\[
P_H^2=P_H,\quad (P_H^\perp)^2=P_H^\perp,\quad P_H^\top=P_H,\quad (P_H^\perp)^\top=P_H^\perp,
\quad\text{and}\quad P_H+P_H^\perp=I_r,\; P_H P_H^\perp=0.
\]
By definition, $Z_\parallel = ZP_H$ and $Z_\perp = ZP_H^\perp$. Therefore,
\[
M \;=\; ZZ^\top \;=\; Z(P_H+P_H^\perp)(P_H+P_H^\perp)^\top Z^\top.
\]
Using symmetry of the projectors and expanding, we obtain
\begin{align*}
M
&= Z(P_H+P_H^\perp)(P_H+P_H^\perp) Z^\top \\
&= ZP_H Z^\top + ZP_H^\perp Z^\top
   + ZP_H P_H^\perp Z^\top + ZP_H^\perp P_H Z^\top \\
&= ZP_H Z^\top + ZP_H^\perp Z^\top
 \;=\; Z_\parallel Z_\parallel^\top + Z_\perp Z_\perp^\top
 \;=\; M_\parallel + M_\perp,
\end{align*}
since $P_H P_H^\perp = P_H^\perp P_H = 0$. This proves $M=M_\parallel+M_\perp$.

Next, the cross term vanishes:
\[
Z_\parallel Z_\perp^\top
\;=\; (ZP_H)(ZP_H^\perp)^\top
\;=\; ZP_H (P_H^\perp)^\top Z^\top
\;=\; ZP_H P_H^\perp Z^\top
\;=\; 0.
\]

We now show $U^\top Z_\perp=0$. Using $G=U^\top Z$ and $Z_\perp = ZP_H^\perp$,
\[
U^\top Z_\perp \;=\; U^\top Z P_H^\perp \;=\; G P_H^\perp.
\]
By definition $H=\mathrm{rowspan}(G)$, hence every row of $G$ lies in $H$. Projecting any
vector in $H$ onto $H^\perp$ yields zero, so $G P_H^\perp=0$, and therefore $U^\top Z_\perp=0$.

Finally, since $M_\perp = Z_\perp Z_\perp^\top$, we have
\[
\mathrm{range}(M_\perp)\subseteq \mathrm{range}(Z_\perp).
\]
Moreover, for any $x\in S$ we can write $x=U\alpha$ for some $\alpha\in\mathbb{R}^s$, and thus
\[
Z_\perp^\top x \;=\; Z_\perp^\top U\alpha \;=\; (U^\top Z_\perp)^\top \alpha \;=\; 0.
\]
Hence
\[
M_\perp x \;=\; Z_\perp Z_\perp^\top x \;=\; Z_\perp (Z_\perp^\top x) \;=\; 0,
\]
which shows $M_\perp a=0$ for all $a\in S$. Equivalently, $\mathrm{range}(M_\perp)\subseteq S^\perp$.
\end{proof}

We define the output variables of interest as
\[
X \coloneqq Mv_j+\xi_j\in\mathbb{R}^d,
\qquad
Y \coloneqq \bigl[\,Mv_k+\xi_k\,\bigr]_{k\neq j}\in\mathbb{R}^{d\times(n-1)}.
\]
where $v_j$ denotes the jth column of $V \in \mathbb{R}^{d \times n}$ and 
$\{\xi_i\}_{i=1}^n$ are independent noise vectors.

\begin{lem}[Posterior stability of the residual block]\label{lem:posterior-stable}
Conditional on $G$ (and hence on $H$ and $P_H$), the random matrix $Z_\perp=ZP_H^\perp$ is independent of $Y$.
Equivalently,
\[
\mathcal{L}(M_\perp\mid G,Y)\;=\;\mathcal{L}(M_\perp\mid G).
\]
\end{lem}
\begin{proof}
Recall the orthogonal decomposition
\[
Z = UG + U_\perp W,
\qquad
G:=U^\top Z\in\mathbb{R}^{s\times r},\quad
W:=U_\perp^\top Z\in\mathbb{R}^{(d-s)\times r},
\]
where $G$ and $W$ are independent and have i.i.d.\ $\mathcal N(0,\sigma_M^2)$ entries.
Let $H=\mathrm{rowspan}(G)$ and let $P_H^\perp$ be the orthogonal projector onto $H^\perp$.
By definition,
\[
Z_\perp = ZP_H^\perp = (UG + U_\perp W)P_H^\perp
       = U(GP_H^\perp) + U_\perp(WP_H^\perp).
\]
Since $H=\mathrm{rowspan}(G)$, every row of $G$ lies in $H$, hence projecting onto $H^\perp$
annihilates the rows of $G$, i.e.\ $GP_H^\perp=0$. Therefore
\begin{equation}\label{eq:Zperp-W}
Z_\perp = U_\perp\, W\, P_H^\perp.
\end{equation}
In particular, conditional on $G$ (and therefore conditional on $H$ and $P_H$), the projector $P_H^\perp$
is deterministic, and \eqref{eq:Zperp-W} shows that $Z_\perp$ is a measurable function of $W$ only.

Next, by the Gaussian block decomposition above, the matrices
\[
G:=U^\top Z \in \mathbb{R}^{s\times r}
\qquad\text{and}\qquad
W:=U_\perp^\top Z \in \mathbb{R}^{(d-s)\times r}
\]
are independent and have i.i.d.\ $\mathcal{N}(0,\sigma_M^2)$ entries.

Recall that $M=ZZ^\top$ and that we defined
\[
Y \coloneqq \bigl[\,Mv_k+\xi_k\,\bigr]_{k\neq j}\in\mathbb{R}^{d\times(n-1)}.
\]
Since $v_k\in S=\mathrm{span}(V_{-j})$ for all $k\neq j$, Lemma~\ref{lem:split} gives
$M_\perp v_k=0$, and therefore
\[
Mv_k = (M_\parallel+M_\perp)v_k = M_\parallel v_k
\qquad\text{for all } k\neq j.
\]
Hence $Y$ can be written as
\[
Y
= \bigl[\,M_\parallel v_k+\xi_k\,\bigr]_{k\neq j}
= \bigl[\,Z_\parallel Z_\parallel^\top v_k+\xi_k\,\bigr]_{k\neq j},
\]
which shows that $Y$ depends on $Z$ only through $Z_\parallel = ZP_H$ and the independent noises
$\{\xi_k\}_{k\neq j}$. 
Moreover, conditional on $G$ (hence on $H$ and $P_H$), we can make the dependence on $W$ explicit.
Since
\[
Z_\parallel = ZP_H = (UG+U_\perp W)P_H = UG + U_\perp(WP_H),
\]
it follows that, conditional on $G$, the random variable $Y$ is measurable with respect to
$\sigma(WP_H,\{\xi_k\}_{k\neq j})$. On the other hand,
\[
Z_\perp = ZP_H^\perp = (UG+U_\perp W)P_H^\perp = U_\perp(WP_H^\perp),
\]
so $Z_\perp$ is measurable with respect to $\sigma(WP_H^\perp)$ (conditional on $G$).
Finally, conditional on $G$, the Gaussian matrix $W$ decomposes as
\[
W = WP_H + WP_H^\perp,
\]
where $WP_H$ and $WP_H^\perp$ are independent. Therefore,
\[
Z_\perp \;\perp\!\!\!\perp\; Y \mid G.
\]




Finally, since $M_\perp = Z_\perp Z_\perp^\top$ is a measurable function of $Z_\perp$,
the same conditional independence carries over:
\[
\mathcal{L}(M_\perp \mid G,Y) \;=\; \mathcal{L}(M_\perp \mid G).
\]
which concludes the proof.

\end{proof}

\privrescomponent*
\begin{proof}
We treat the two cases separately.

\paragraph{Case 1: $\beta_\perp = 0$.}
We defined $\|b\| = \beta_\bot$ therefore $\beta_\perp=0$ means $b$=0
\[
R(b)=M_\perp b = 0
\qquad\text{and similarly}\qquad
R(b')=M_\perp b' = 0
\]
deterministically.

\paragraph{Case 2: $\beta_\perp > 0$.}
Let $G=U^\top Z$ and recall that $H=\mathrm{rowspan}(G)$ and $Z_\perp = ZP_H^\perp$.
By Lemma~\ref{lem:posterior-stable}
\[
\mathcal{L}(M_\perp \mid G,Y) \;=\; \mathcal{L}(M_\perp \mid G),
\]
therefore it suffices to prove that, conditional on $G$, the map $b\mapsto M_\perp b$
is $(\varepsilon_\perp,\delta_\perp)$-DP.

Fix $G$ (equivalently, fix $H$ and the projector $P_H^\perp$). Using the decomposition
$Z = UG + U_\perp W$ and the fact that $GP_H^\perp=0$, we have
\[
Z_\perp \;=\; ZP_H^\perp \;=\; (UG+U_\perp W)P_H^\perp \;=\; U_\perp\,W\,P_H^\perp.
\]
Hence
\[
M_\perp
\;=\;
Z_\perp Z_\perp^\top
\;=\;
U_\perp\,(WP_H^\perp)(WP_H^\perp)^\top\,U_\perp^\top.
\]
Consequently, for any $b\in\mathbb{R}^d$,
\begin{equation}\label{eq:R-factorization}
R(b)=M_\perp b
\;=\;
U_\perp\,(WP_H^\perp)(WP_H^\perp)^\top\,\underbrace{(U_\perp^\top b)}_{=: \,b_\perp}.
\end{equation}

\paragraph{Step 1: identify the residual mechanism.}
Conditional on $G$, the projector $P_H^\perp$ is deterministic and $W$ remains a Gaussian matrix with i.i.d.\ $\mathcal{N}(0,\sigma_M^2)$ entries. 
Conditional on $G$, the subspace $H=\mathrm{rowspan}(G)$ is fixed, and hence the projector
$P_H^\perp$ is deterministic. Choose an orthonormal basis $Q_\perp\in\mathbb{R}^{r\times(r-p)}$
for $H^\perp$, so that
\[
P_H^\perp = Q_\perp Q_\perp^\top.
\]
Then
\[
WP_H^\perp = WQ_\perp Q_\perp^\top.
\]
Since $W$ has i.i.d.\ $\mathcal{N}(0,\sigma_M^2)$ entries and $Q_\perp$ has orthonormal columns,
the matrix $WQ_\perp\in\mathbb{R}^{(d-s)\times(r-p)}$ has i.i.d.\ $\mathcal{N}(0,\sigma_M^2)$ entries.
Moreover,
\[
(WP_H^\perp)(WP_H^\perp)^\top
= (WQ_\perp Q_\perp^\top)(Q_\perp Q_\perp^\top W^\top)
= (WQ_\perp)(WQ_\perp)^\top.
\]
Therefore, conditional on $G$, the random matrix
\[
\widetilde M_\perp \;:=\; (WP_H^\perp)(WP_H^\perp)^\top
\]
has the same distribution as a (scaled) Wishart matrix in dimension $(d-s)$ with $(r-p)$ degrees of freedom, where $p=\dim(H)=\mathrm{rank}(G)$. 

\paragraph{Step 2: apply the vector-DP guarantee and post-processing.}
Consider the ``core'' residual mechanism
\[
\widetilde R(b_\perp) \;:=\; \widetilde M_\perp\, b_\perp
\qquad \in \mathbb{R}^{d-s}.
\]
By~\Cref{thm:vec-alt}, we obtain that conditional on $G$,
the map $b_\perp\mapsto \widetilde R(b_\perp)$ is
\[
\big(\varepsilon_{\mathrm{vec}}(\rho_\perp; d-s, r-p),\ \delta_{\mathrm{vec}}(\rho_\perp; d-s, r-p)\big)\text{-DP}.
\]
Finally, \eqref{eq:R-factorization} shows that $R(b)$ is obtained from $\widetilde R(b_\perp)$
by applying the deterministic linear map $x\mapsto U_\perp x$ (given $G$). Since
differential privacy is preserved under post-processing, it follows that conditional on $G$,
the map $b\mapsto R(b)=M_\perp b$ is $(\varepsilon_\perp,\delta_\perp)$-DP. The exact $(\varepsilon_\bot, \delta_\bot)$ is then obtained by instantiating~\Cref{thm:vec-alt} with a Wishart matrix of dimension $(d-s)\times(r-p)$ and alignment parameter $\rho_\bot$.
\end{proof}

\begin{restatable}{lem}{boundonMparallel}\label{lem:dir-bound}
Fix $\beta\in(0,1)$ and set $g_\beta=\sqrt{2\ln(2/\beta)}$. Conditional on $H$, with probability at least $1-\beta$ over the draw of $Z$,
\begin{equation}
\label{eq:dir-Mpar}
\norm{M_\parallel u}
\le
\sigma_M^2\br{\sqrt d+\sqrt p+g_\beta}\br{\sqrt p+g_\beta}
\coloneqq
\Gamma_\beta.
\end{equation} 
\end{restatable}

\begin{proof}
Let $P_H$ be the orthogonal projector onto $H$ and recall that
\[
Z_\parallel = ZP_H,
\qquad
M_\parallel = Z_\parallel Z_\parallel^\top.
\]
Fix a unit vector $u\in\mathbb{R}^d$ (the bound scales by $\|u\|$ otherwise).  Conditional on $H$,
choose an orthonormal basis matrix $Q\in\mathbb{R}^{r\times p}$ for $H$ so that
\[
P_H = QQ^\top,
\qquad Q^\top Q = I_p.
\]
Define the $d\times p$ Gaussian matrix
\[
\widetilde Z \;\coloneqq\; ZQ.
\]
Then
\[
M_\parallel
= ZP_H Z^\top
= ZQQ^\top Z^\top
= (ZQ)(ZQ)^\top
= \widetilde Z\,\widetilde Z^\top,
\]
and hence
\begin{equation}\label{eq:Mpar-factor}
\|M_\parallel u\|
= \|\widetilde Z\,\widetilde Z^\top u\|
\le \|\widetilde Z\|_{\mathrm{op}} \cdot \|\widetilde Z^\top u\|.
\end{equation}

Since each row of $Z$ is distributed as $\mathcal N(0,\sigma_M^2 I_r)$ and $Q$ has orthonormal
columns, we have for each row $z_i^\top$ of $Z$,
\[
(z_i^\top Q)^\top \sim \mathcal N(0,\sigma_M^2 I_p).
\]
Rows remain independent, hence conditional on $H$,
$\widetilde Z\in\mathbb{R}^{d\times p}$ has i.i.d.\ $\mathcal N(0,\sigma_M^2)$ entries.

Let $G\in\mathbb{R}^{d\times p}$ have i.i.d.\ $\mathcal N(0,1)$ entries so that
$\widetilde Z = \sigma_M G$. A standard Gaussian operator norm bound gives that for all $t\ge 0$,
\[
\mathbb{P}\Big(\|G\|_{\mathrm{op}} \ge \sqrt d+\sqrt p+t\Big) \le e^{-t^2/2}.
\]
Moreover, since $u$ is fixed and $\|u\|=1$, we have $G^\top u \sim \mathcal N(0,I_p)$, and thus
\[
\mathbb{P}\Big(\|G^\top u\| \ge \sqrt p+t\Big) \le e^{-t^2/2}.
\]
Set $t=g_\beta=\sqrt{2\ln(2/\beta)}$. Then $e^{-t^2/2}=\beta/2$, and scaling back by $\sigma_M$ yields
\begin{align*}
\mathbb{P}\Big(\|\widetilde Z\|_{\mathrm{op}} \le \sigma_M(\sqrt d+\sqrt p+g_\beta)\,\Big|\,H\Big)
&\ge 1-\beta/2,\\
\mathbb{P}\Big(\|\widetilde Z^\top u\| \le \sigma_M(\sqrt p+g_\beta)\,\Big|\,H\Big)
&\ge 1-\beta/2.
\end{align*}
By a union bound, with conditional probability at least $1-\beta$ (given $H$), both events hold.

All together this mean with probability at least $1- \beta$ (given $H$) we have
\[
\|M_\parallel u\|
\le
\|\widetilde Z\|_{\mathrm{op}} \cdot \|\widetilde Z^\top u\|
\le
\sigma_M^2(\sqrt d+\sqrt p+g_\beta)(\sqrt p+g_\beta)
=
\Gamma_\beta.
\]
\end{proof}

\begin{lem}\label{lem:U-dp-dir}
    Fix \(\beta\in\br{0,1}\) and $\delta_\parallel \in(0, 1)$. Consider the mechanism $C(v) = M_\parallel v + \xi$, where $\xi\sim \cN(0, \sigma_G^2 I_d)$ is independent of \(M_\parallel\). Suppose neighbouring inputs satisfy \(\norm{v-v'}\leq \Delta_v\). Then, \(C\) is $(\varepsilon_\parallel, \delta_\parallel+\beta)$-DP, where \[
\varepsilon_\parallel \;=\;
\frac{\Gamma_\beta\Delta_v}{\sigma_G}\sqrt{2\ln\!\Big(\frac{1.25}{\delta_\parallel}\Big)},
\]
where $\Gamma_\beta$ is as in \Cref{lem:dir-bound}.
\end{lem}
\begin{proof}
Fix any neighbouring \(v\sim v'\) and define \(u=\frac{v-v'
}{\norm{v-v'}}\).
Let $\cE_{u,\beta}$ be the event 
\[
\cE_{u,\beta} :=\bc{\norm{M_\parallel u}\le \Gamma_\beta }.
\]
By \Cref{lem:dir-bound}, we have $\Pr\br{\cE_{u,\beta}}\ge 1-\beta$.

Conditional on $M_\parallel$, the outputs are Gaussians
\[
C(v)\mid M_\parallel \sim \cN\br{M_\parallel v,\sigma_G^2 I_d},
\qquad
C(v')\mid M_\parallel \sim \cN\br{M_\parallel v',\sigma_G^2 I_d},
\]
whose means differ by
\[
M_\parallel\br{v-v'} = \norm{v-v'} M_\parallel u.
\]

On \(\cE_{u,\beta}\) we have \[\norm{M_\parallel\norm{v-v'} u}\leq \Gamma_\beta \Delta_v\]

Therefore, on \(\cE_{u,\beta}\) the standard Gaussian mechanism analysis implies that \(C(\cdot)\) is \((\varepsilon_\parallel,\delta_\parallel)\)-DP with
\[
\varepsilon_\parallel = \frac{\Gamma_\beta\Delta_v}{\sigma_G}\sqrt{2\ln\br{\frac{1.25}{\delta_\parallel}}}.
\]

Finally, remove the conditioning: for any measurable set \(S\subseteq \reals^d\),
\begin{align*}
\Pr\!\br{C(v)\in S}
&\le \Pr\!\br{C(v)\in S \mid \cE_{u,\beta}}\Pr\!\br{\cE_{u,\beta}}
  + \Pr\!\br{\cE_{u,\beta}^c} \\
&\le \Bigl(e^{\varepsilon_{\parallel}}\Pr\!\br{C(v')\in S \mid \cE_{u,\beta}} + \delta_{\parallel}\Bigr)\Pr\!\br{\cE_{u,\beta}}
  + \Pr\!\br{\cE_{u,\beta}^c} \\
&\le e^{\varepsilon_{\parallel}}\Pr\!\br{C(v')\in S} + \delta_{\parallel} + \Pr\!\br{\cE_{u,\beta}^c} \\
&\le e^{\varepsilon_{\parallel}}\Pr\!\br{C(v')\in S} + \delta_{\parallel} + \beta.
\end{align*}

\end{proof}

\privmatrixprojonecol*
\begin{proof}
We can write the the mechanism $\cA(V)=MV+\Xi$ as \[(X, Y) = (Mv_j+\xi_j\in\reals^d, \bs{Mv_k+\xi_k}_{k\neq j}\in\reals^{d\times(n-1)}) = (M_\parallel v_j + \xi_j + M_\perp b, \bs{M_\parallel v_k+\xi_k}_{k\neq j}\in\reals^{d\times(n-1)} ).\] So if we define \[ \begin{aligned}
    R(b) &= M_\perp b \\ 
    C(v) &= M_\parallel v_j
\end{aligned}
\]
we know by~\Cref{lem:split}, conditional on $H$, the residual randomness $M_\perp$ is independent of $Y$. Therefore,~\Cref{lem:resid-dp} tells us that conditional on $(H,Y)$ $R$ is $(\varepsilon_\perp, \delta_\perp)$-DP. 

For the correlated term $C$, \Cref{lem:U-dp-dir} establishes $(\varepsilon_\parallel, \delta_\parallel)$-differential privacy using the Gaussian mechanism with the directional sensitivity bound from \Cref{lem:dir-bound}.

By composition and the postprocessing Lemma this tells us that conditioned on $H$ and $\|M_\parallel u\| \leq \Gamma_\beta$ $(X,Y)$ is $(\varepsilon_\parallel + \varepsilon_\bot, \delta_\parallel + \delta_\bot)$-DP. We can remove the conditioning on $\|M_\parallel u\| \leq \Gamma_\beta$ by adding an additional $\beta$ to our final $\delta$. The conditioning on $H$ can be removed because $H$ is a measurable function of $(Z,V_{-j})$.
Under $V\sim_A^{(j)}V'$, we have $V_{-j}=V'_{-j}$, so the law of $H$ is the same under $V$ and $V'$.
Therefore conditioning on $H$ does not affect the DP comparison.
\end{proof}

\begin{proof}
Write the mechanism as $\cA(V)=MV+\Xi$, and denote its $j$th column by $X \coloneqq Mv_j+\xi_j$ and the remaining columns by $Y \coloneqq \bigl[\,Mv_k+\xi_k\,\bigr]_{k\neq j}$. 

Using the orthogonal split $M=M_\parallel+M_\perp$ and Lemma~\ref{lem:split}, we have
$M_\perp v_k=0$ for all $k\neq j$. Hence
\[
Y = \bigl[\,M_\parallel v_k+\xi_k\,\bigr]_{k\neq j},
\qquad
X = M_\parallel v_j+\xi_j + M_\perp v_j.
\]
Define the residual map $R(b)\coloneqq M_\perp b$ and the "main" (correlated) map
\[
\cC(v) \coloneqq M_\parallel v+\xi_j.
\]
We will first argue about the privacy of $X$ conditioned on $Y$ and in a final step remove the conditioning.

By Lemma~\ref{lem:posterior-stable}, conditional on $H$ the residual block $M_\perp$ is independent of $Y$. So conditioned on $(H,Y)$ $M_\perp$ is distributed like a random Wishart random matrix which is what Lemma~\ref{lem:resid-dp} exploits to show that conditional on $(H,Y)$, the map $b \mapsto R(b)=M_\perp b$ is $(\varepsilon_\perp,\delta_\perp)$-DP. 

By Lemma~\ref{lem:U-dp-dir}, conditional on $(H,Y)$, the mechanism
\[
v \mapsto \cC(v)=M_\parallel v+\xi_j
\]
is $(\varepsilon_\parallel,\delta_\parallel+\beta)$-DP with
\[
\varepsilon_\parallel
=
\frac{\Gamma_\beta \Delta_v}{\sigma_G}\sqrt{2\ln\!\Big(\frac{1.25}{\delta_\parallel}\Big)}.
\]

Conditional on $(H,Y)$ and on $\mathsf{E}_\beta$, the release of $X$ can be written as the composition of two DP mechanisms
\[
X = \cC(v_j) + R(v_j),
\]
is therefore by sequential composition $(\varepsilon_\parallel+\varepsilon_\perp,\delta_\parallel+\delta_\perp)$-DP
conditional on $(H,Y, \mathsf{E}_\beta)$.

By~\Cref{lem:dir-bound}, $\mathbb{P}(\mathsf{E}_\beta^c\mid H)\le\beta$.
Thus the same mechanism is $(\varepsilon_\parallel+\varepsilon_\perp,\delta_\parallel+\delta_\perp+\beta)$-DP conditional on $(H,Y)$. 

Since $V\sim_A^{(j)}V'$ implies $V_{-j}=V'_{-j}$, the random variables $H$ and $Y$ are distributed
the same under $V$ and $V'$. Therefore, a conditional DP guarantee for $X$ given $(H,Y)$ implies that the
joint release $(X,Y)$ is $(\varepsilon,\delta)$-DP with
\[
\varepsilon \le \varepsilon_\parallel+\varepsilon_\perp,
\qquad
\delta \le \delta_\parallel+\delta_\perp+\beta.
\]
\end{proof}

\subsection{Proofs small r regime}
For
\begin{align*}
    Y &= M(V + \sigma E)\\
    Y' &= M(V' + \sigma E) \\
    M &= \sum_i^r z_i z_i^\top \text{ with } z_i \sim \mathcal{N}(0, I_d)\\
    \Delta V &= V - V'
\end{align*}
let $P$ be the law of $Y$ and $Q$ the law of $Y'$. Then because $P$ and $Q$ are mutually absolute continuous we are able to define the the density ratio.

\begin{lem}\label{lem:conditional-priv-loss}
   For $L_M(Y)$ defined as \[ L_M(Y) := \frac{dP(\cdot|M)}{dQ(\cdot|M)}(Y)\] we have \[\log L_M(Y)|M \sim \mathcal{N} \left( -\frac{\|P_M \Delta V \|_F^2}{2 \sigma^2}, \frac{\|P_M \Delta V \|_F^2}{ \sigma^2} \right) \]
\end{lem}

\begin{proof}
    Since $Y|M$ is a Gaussian and any affine function of a Gaussian is Gaussian (and the log likelihood ratio is affine) we have that $\log L_M(Y) | M$ is Gaussian. So we only need to determine its mean and variance. For $Y | M \sim \mathcal{N}(\mu, \Sigma)$ and $Y' | M \sim \mathcal{N}(\mu', \Sigma)$ we have that \[ \log L_M(y)= (\mu - \mu')^\top \Sigma^{\dagger} (y - \frac{\mu + \mu'}{2})\] where $\Sigma^{\dagger}$ is the pseudoinverse. 

Mean: 
\begin{align*}
    \mathbb{E}[\log L_M(Y)|M ] = \mathbb{E}[(\mu - \mu')^\top \Sigma^{\dagger} (Y - \frac{\mu + \mu'}{2}) |M] = (\mu - \mu')^\top \Sigma^{\dagger} (\mathbb{E}[ Y | M] - \frac{\mu + \mu'}{2}) = \frac{1}{2} (\mu - \mu')^\top \Sigma^{\dagger} (\mu - \mu')^\top
\end{align*}
where the last step follows as by the definition of the log-likelihood that $Y$ is distributed by the nominator. (The inverse log likelihood would lead to a minus sign in the mean)

Variance:
\begin{align*}
    \text{Var}(\log L_M(Y)|M) &= \text{Var}( (\mu - \mu')^\top \Sigma^{\dagger} Y|M) = \mathbb{E}[((\mu - \mu')^\top \Sigma^{\dagger} Y)^2 |M] - \mathbb{E}[(\mu - \mu')^\top \Sigma^{\dagger} Y|M]^2 \\ &= \mathbb{E}[(\mu - \mu')^\top \Sigma^{\dagger} Y Y^\top \Sigma^{\dagger}(\mu - \mu') |M] - ((\mu - \mu')^\top \Sigma^{\dagger} \mu)((\mu - \mu')^\top \Sigma^{\dagger} \mu)^\top \\ &= (\mu - \mu')^\top \Sigma^{\dagger}\mathbb{E}[Y Y^\top |M] \Sigma^{\dagger}(\mu - \mu') - (\mu - \mu')^\top \Sigma^{\dagger} \mu \mu^\top \Sigma^{\dagger}(\mu - \mu') \\ &= (\mu - \mu')^\top \Sigma^{\dagger} (\mathbb{E}[Y Y^\top |M] - \mu \mu^\top) \Sigma^{\dagger}(\mu - \mu') \\ &= (\mu - \mu')^\top \Sigma^{\dagger} \text{Var}(Y|M) \Sigma^{\dagger}(\mu - \mu') \\ &= (\mu - \mu')^\top \Sigma^{\dagger} \Sigma  \Sigma^{\dagger}(\mu - \mu') = (\mu - \mu')^\top \Sigma^{\dagger}(\mu - \mu')
\end{align*}

Now let's recall that 
\begin{align*}
    Y &= MV + \sigma M E \\
    Y' &= MV' + \sigma M E
\end{align*}
so the jth column is distributed as
\begin{align*}
    Y_{:, j}|M &\sim \mathcal{N}(MV_{:, j}, \sigma^2 M^2) \\
    Y'_{:, j}|M &\sim \mathcal{N}(MV_{:, j}', \sigma^2 M^2)
\end{align*}
and the columns are independent given $M$. So if we define $y = \text{vec}(Y)$ (stack columns into one vector) we have that $\mu = \text{vec}(MV), \mu' = \text{vec}(MV')$ and because columns are independent given $M$, $\text{Cov}( Y_{:,i}, Y_{:,j}|M) = 0$ for $i \neq j$  which means \[
\Sigma \;=\; \operatorname{Cov}(y \mid M)
\;=\;
\begin{pmatrix}
\sigma^{2} M^{2} & 0 & \cdots & 0 \\
0 & \sigma^{2} M^{2} & \cdots & 0 \\
\vdots & \vdots & \ddots & \vdots \\
0 & 0 & \cdots & \sigma^{2} M^{2}
\end{pmatrix}
\;=\;
\sigma^{2}\,\operatorname{diag}\!\bigl(M^{2},\ldots,M^{2}\bigr) = \sigma^{2}\bigl(I_{k}\otimes M^{2}\bigr).
\] 
Recall we need the pseudoinverse of our variance, for which we will use two identities:
\begin{align*}
    (\alpha A)^{\dagger} &= \frac{1}{\alpha} A^{\dagger} \\
    ( A \otimes B)^{\dagger} &= A^{\dagger} \otimes B^{\dagger}
\end{align*}
so we get
\[ \Sigma^{\dagger} = \frac{1}{\sigma^2} ( I_k \otimes (M^2)^{\dagger}).\]
Then by the Kronecker-Vec Identity we have \[
(\Delta \mu)^\top \Sigma^\dagger (\Delta \mu)
\;=\;
\frac{1}{\sigma^{2}}\,
\operatorname{tr}\!\left((M\Delta V)^\top (M^{2})^\dagger (M\Delta V)\right).
\]
Using $P_M = M(M^2)^{\dagger} M$ and $P_M^2 = P_M$ we finally get \[ (\Delta \mu)^\top \Sigma^\dagger (\Delta \mu) = \|P_M \Delta V\|_F^2. \] Plugging this into the mean/variance formulas we get
\[ 
 \log L_M(Y)|M \sim \mathcal{N}( - \frac{\|P_M \Delta V\|_F^2}{2 \sigma^2}, \frac{\|P_M \Delta V\|_F^2}{2 \sigma^2})
\]
\end{proof}

\begin{lem}[Tail bound for Haar projection]\label{lem:delta-M-bound}
Let $r\in\{1,\dots,d-1\}$ and let $Z\in\R^{d\times r}$ have i.i.d.\ $\cN(0,1)$ entries.
Let $P_M$ denote the orthogonal projector onto $\mathrm{col}(Z)$ (equivalently onto
$\mathrm{col}(M)$ for $M=ZZ^\top$). Fix a deterministic matrix $\Delta V\in\R^{d\times n}$
of rank $s\ge 1$, and write $\|\cdot\|_F$ for the Frobenius norm. Then for every
$\alpha\in(0,1)$,
\[
\Pr\!\left(\frac{\|P_M \Delta V\|_F^2}{\|\Delta V\|_F^2}>\alpha\right)
\;\le\;
s\left[1-I_{\alpha}\!\left(\frac r2,\frac{d-r}{2}\right)\right],
\]
where $I_{\alpha}(a,b)$ is the regularized incomplete beta function (the $\mathrm{Beta}(a,b)$ CDF).
\end{lem}

\begin{remark}
In our setting the matrix \(Z\) is generated with i.i.d.\ entries \(Z_{ij}\sim\cN(0,1/r)\).
This differs from the standard \(Z_{ij}\sim\cN(0,1)\) only by a scalar factor:
\(Z=\frac{1}{\sqrt r}G\) with \(G_{ij}\sim\cN(0,1)\). Since scaling by a nonzero constant does not
change the column space, \(\mathrm{col}(Z)=\mathrm{col}(G)\), and hence the orthogonal projector
\(P_M\) onto \(\mathrm{col}(Z)\) has the same distribution. Therefore \Cref{lem:delta-M-bound}
applies unchanged.
\end{remark}

\begin{proof}

Because \(Z\) has i.i.d.\ standard normal entries, its law is orthogonally invariant:
\[
UZ \;\overset{d}{=}\; Z \qquad \text{for all } U\in O(d).
\]
Therefore \(\mathrm{col}(Z)\) has a rotation-invariant distribution on the Grassmannian
\(\mathrm{Gr}(d,r)\) (see \Cref{def:grassmannian,def:uniform-grassmannian}), hence it is uniform.
Moreover, in the (thin) QR decomposition \(Z=QR\), the factor \(Q\in V_{d,r}\) 
inherits the same invariance
\[
UQ \;\overset{d}{=}\; Q \qquad \text{for all } U\in O(d),
\]
and is thus Haar-uniform on the Stiefel manifold in the sense of \Cref{def:haar-stiefel}. Since \(P_M = QQ^\top\) and \(Q\) is Haar-uniform on \(V_{d,r}\),
for any fixed \(U\in O(d)\) we have
\[
U P_M U^\top \;=\; UQQ^\top U^\top \;=\; (UQ)(UQ)^\top.
\]
By Haar-uniformity, \(UQ \overset{d}{=} Q\), hence
\[
U P_M U^\top \;\overset{d}{=}\; QQ^\top \;=\; P_M.
\]
Thus the law of \(P_M\) is invariant under conjugation by orthogonal matrices. This is exactly the
notion of Haar-uniformity for rank-\(r\) orthogonal projectors.

Assume $\Delta V$ has rank $1$, so $\Delta V = u w^\top$ with $\|u\|_2=1$.
Then
\[
\frac{\|P_M\Delta V\|_F^2}{\|\Delta V\|_F^2}
=
\frac{\|P_M u w^\top\|_F^2}{\|u w^\top\|_F^2}
=
\frac{\|P_M u\|_2^2\|w\|_2^2}{\|u\|_2^2\|w\|_2^2}
=
\|P_M u\|_2^2.
\]
Let 
\[
P_0 \;\coloneqq\; \begin{pmatrix} I_r & 0 \\ 0 & 0 \end{pmatrix}.
\]
Since every rank-\(r\) orthogonal projector is an orthogonal conjugate of \(P_0\), and \(P_M\) is Haar-uniform,
we may write \(P_M \overset{d}{=} U P_0 U^\top\) with \(U\sim\mathrm{Haar}(O(d))\).
Then \(y \coloneqq U^\top u\) is uniform on the unit sphere \(S^{d-1}\), and 
\[
X := \|P_M u\|_2^2 = \sum_{i=1}^r y_i^2.
\]
Let $g\sim\cN(0,I_d)$ and note $g/\|g\|_2$ is uniform on $S^{d-1}$ .
Write $g=(g_{1:r},g_{r+1:d})$. Then
\[
X \stackrel{d}= \frac{\sum_{i=1}^r g_i^2}{\sum_{i=1}^d g_i^2}
=
\frac{U}{U+V},
\qquad
U\sim\chi^2_r,\; V\sim\chi^2_{d-r}\ \text{ independent.}
\]
Hence $X\sim \mathrm{Beta}\!\left(\frac r2,\frac{d-r}{2}\right)$.
Therefore,
\[
\Pr(X>\alpha)=1-I_{\alpha}\!\left(\frac r2,\frac{d-r}{2}\right).
\]

For a $\Delta V$ with general rank, by SVD we obtain $\Delta V=\sum_{j=1}^s \sigma_j u_j v_j^\top$ with orthonormal $\{u_j\}_{j=1}^s$.
Using $P_M^\top P_M=P_M$ and orthonormality,
\[
\|P_M\Delta V\|_F^2
=
\sum_{j=1}^s \sigma_j^2\|P_M u_j\|_2^2,
\qquad
\|\Delta V\|_F^2=\sum_{j=1}^s \sigma_j^2.
\]
Define weights $w_j:=\sigma_j^2/\sum_{\ell=1}^s\sigma_\ell^2$ and
$X_j:=\|P_M u_j\|_2^2\in[0,1]$. Then
\[
\frac{\|P_M\Delta V\|_F^2}{\|\Delta V\|_F^2}=\sum_{j=1}^s w_j X_j.
\]
If $\sum_{j=1}^s w_j X_j>\alpha$ and $\sum_j w_j=1$ with $w_j\ge 0$, then necessarily
$\max_{1\le j\le s} X_j>\alpha$ (otherwise all $X_j\le \alpha$ would imply the weighted
average is $\le\alpha$). Thus,
\[
\Pr\!\left(\sum_{j=1}^s w_j X_j>\alpha\right)
\le
\Pr\!\left(\max_{1\le j\le s} X_j>\alpha\right)
\le
\sum_{j=1}^s \Pr(X_j>\alpha).
\]
Finally, each $X_j$ has the same marginal law as in Step 2 because $u_j$ is a fixed unit
vector and $P_M$ is Haar, so $\Pr(X_j>\alpha)=1-I_{\alpha}\!\left(\frac r2,\frac{d-r}{2}\right)$.
Therefore,
\[
\Pr\!\left(\frac{\|P_M \Delta V\|_F^2}{\|\Delta V\|_F^2}>\alpha\right)
\le
s\left[1-I_{\alpha}\!\left(\frac r2,\frac{d-r}{2}\right)\right],
\]
as claimed.
\end{proof}

\smallrmatrixdpprojmech*

\begin{proof}
First note that
\begin{align*}
    \mathbb{P}(Y \in A) &= \mathbb{E}_M[\mathbb{P}(Y \in A|M)] = \mathbb{E}_M[\mathbb{P}(Y \in A|M) \mathbf{1}_{ \{M \in \mathcal{G}_{\alpha} \}}] + \mathbb{E}_M[\mathbb{P}(Y \in A|M) \mathbf{1}_{ \{M \notin \mathcal{G}_{\alpha} \}}] \\ &\leq \mathbb{E}_M[\mathbb{P}(Y \in A|M) \mathbf{1}_{ \{M \in \mathcal{G}_{\alpha} \}}] + \delta_M 
\end{align*}

So we can analyse $\mathbb{E}_M[\mathbb{P}(Y \in A|M) \mathbf{1}_{ \{M \in \mathcal{G}_{\alpha} \}}]$ separately and find $\epsilon, \delta$ so that

\[
\mathbb{E}\!\left[ \mathbb{P}(Y\in A \mid M)\,\mathbf{1}_{\{M\in G_\alpha\}} \right]
\;\le\;
e^{\varepsilon}\,
\mathbb{E}\!\left[ \mathbb{P}(Y'\in A \mid M)\,\mathbf{1}_{\{M\in G_\alpha\}} \right]
\;+\;
\delta_{E}(\varepsilon,\alpha)\,\mathbb{P}(M\in G_\alpha).
\]

\noindent
Bound \(\mathbb{P}(M\in G_\alpha)\le 1\) and note that
\[
\mathbb{E}\!\left[\mathbb{P}(Y'\in A \mid M)\,\mathbf{1}_{\{M\in G_\alpha\}}\right]
\;\le\;
\mathbb{E}\!\left[\mathbb{P}(Y'\in A \mid M)\right]
\;=\;
\mathbb{P}(Y'\in A).
\]
Therefore,
\[
\mathbb{E}\!\left[\mathbb{P}(Y\in A \mid M)\,\mathbf{1}_{\{M\in G_\alpha\}}\right]
\;\le\;
e^{\varepsilon}\,\mathbb{P}(Y'\in A)
\;+\;
\delta_{E}(\varepsilon,\alpha).
\]

\noindent
Combine with the \(\delta_M\) bound for the complement to obtain
\[
\mathbb{P}(Y\in A)
\;\le\;
e^{\varepsilon}\,\mathbb{P}(Y'\in A)
\;+\;
\delta_{E}(\varepsilon,\alpha)
\;+\;
\delta_M.
\]
Next~\Cref{lem:delta-M-bound} gives us a bound on $\delta_M$ and finally Fix \(\varepsilon>0\). For each fixed \(M\), define the conditional ``good output set''
\[
\mathcal{Y}_{\varepsilon}(M)
\;:=\;
\bigl\{\, y : \lvert \log L_{M}(y)\rvert \le \varepsilon \,\bigr\}.
\]
On \(\mathcal{Y}_{\varepsilon}(M)\) we have the pointwise bound
\[
e^{-\varepsilon}\;\le\; L_{M}(y)\;\le\; e^{\varepsilon}.
\]

Moreover, since \(\ell_{M}(Y)\) is Gaussian as above, we can write its two-sided tail exactly in terms of the standard normal CDF \(\Phi\):
\[
\mathbb{P}\!\left( Y \notin \mathcal{Y}_{\varepsilon}(M)\mid M \right)
\;=\;
\mathbb{P}\!\left( \lvert \ell_{M}(Y)\rvert > \varepsilon \mid M \right)
\;=\;
\Phi\!\left(\frac{-\varepsilon-\mu(M)/2}{\sqrt{\mu(M)}}\right)
\;+\;
1-\Phi\!\left(\frac{\varepsilon-\mu(M)/2}{\sqrt{\mu(M)}}\right),
\]
with the convention that if \(\mu(M)=0\) then this probability equals \(0\) (indeed
\(\ell_{M}(Y)=0\) almost surely).

Now fix a parameter \(\alpha\in(0,1]\) and define the alignment-good event
\[
G_{\alpha}
\;:=\;
\bigl\{\, M : \|P_{M}\Delta V\|_{F}^{2} \le \alpha\,\|\Delta V\|_{F}^{2} \,\bigr\}.
\]
On \(G_{\alpha}\), we have the uniform bound
\[
\mu(M)\;\le\;\bar{\mu},
\qquad\text{where}\qquad
\bar{\mu}
\;:=\;
\frac{\alpha\,\|\Delta V\|_{F}^{2}}{\sigma^{2}}.
\]

Since the tail expression above is increasing in \(\mu(M)\) for the relevant regime,
we can upper bound it by the same expression with \(\mu(M)\) replaced by \(\bar{\mu}\).
Define
\[
\delta_{E}(\varepsilon,\alpha)
\;:=\;
\Phi\!\left(\frac{-\varepsilon-\bar{\mu}/2}{\sqrt{\bar{\mu}}}\right)
\;+\;
1-\Phi\!\left(\frac{\varepsilon-\bar{\mu}/2}{\sqrt{\bar{\mu}}}\right),
\qquad
\bar{\mu}
=
\frac{\alpha\,\|\Delta V\|_{F}^{2}}{\sigma^{2}}.
\]
Then, for all \(M\in G_{\alpha}\),
\[
\mathbb{P}\!\left( Y \notin \mathcal{Y}_{\varepsilon}(M)\mid M \right)
\;\le\;
\delta_{E}(\varepsilon,\alpha).
\]
\end{proof}

\begin{restatable}{corollary}{corsmallrdeltaimprovement}\label{cor:small-r-delta-improvement}
Fix $\varepsilon>0$. Suppose that $r$ satisfies the
scaling regime
\[
\log s \;\lesssim\; r \;\ll\; d.
\]
Then there exists a choice of $\alpha$ on the order of,
\[
\alpha \;\approx\; \frac{r}{d},
\]
such that the privacy bound from \Cref{thm:matrix-dp-proj-mechanism-small-r} is strictly
smaller than the Gaussian baseline, i.e.,
\[
\delta_{\mathrm{ours}}(\varepsilon)\;<\;\delta_{\mathrm{Gauss}}(\varepsilon).
\]
\end{restatable}
\begin{remark}
The lower bound $r\gtrsim \log s$ controls the prefactor $s$ in the Beta-tail term (introduced via a union bound over the $s$ sensitive directions). The condition is likely conservative and a sharper control of the $s$-dimensional subspace could reduce the $\log s$ requirement.
\end{remark}

\begin{proof}
Fix $\varepsilon>0$ and let $\mu=\|\Delta V\|_F^2/\sigma^2$, so that
\[
\delta_{\mathrm{Gauss}}(\varepsilon)=T(\varepsilon;\mu)
:= \Phi\!\left(\frac{-\varepsilon-\mu/2}{\sqrt{\mu}}\right)
\;+\;
1-\Phi\!\left(\frac{\varepsilon-\mu/2}{\sqrt{\mu}}\right).\]
We compare this baseline to the small-$r$ bound from \Cref{thm:matrix-dp-proj-mechanism-small-r},
namely, for any $\alpha\in(0,1)$,
\[
\delta_{\mathrm{ours}}(\varepsilon;\alpha)
\;\le\;
T(\varepsilon;\alpha\mu)
+
s\left[1-I_\alpha\!\left(\frac r2,\frac{d-r}{2}\right)\right].
\]

The map $x\mapsto T(\varepsilon;x)$ is continuous and strictly increasing on $x\ge 0$, and satisfies $T(\varepsilon;0)=0$ and $T(\varepsilon;\mu)=\delta_{\mathrm{Gauss}}(\varepsilon)$. Hence, by the intermediate value theorem, there exists $\alpha_0\in(0,1)$ such that \begin{equation}\label{eq:alpha0-def} T(\varepsilon;\alpha_0\mu) \;=\; \frac12\,T(\varepsilon;\mu) \;=\; \frac12\,\delta_{\mathrm{Gauss}}(\varepsilon). \end{equation}

Fix $\eta\in(0,1)$ (e.g., $\eta=\tfrac12$) and set
\[
\alpha \;\coloneqq\; (1+\eta)\frac{r}{d}.
\]
Assume additionally that $r\le d/2$, so that $\alpha\in(0,1)$ and Lemma~\ref{lem:beta-tail} applies.
Then
\begin{equation}\label{eq:tail-bound-use}
s\left[1-I_\alpha\!\left(\frac r2,\frac{d-r}{2}\right)\right]
\;=\; s\,\Pr(B>\alpha)
\;\le\;
2s\exp\!\left(-\frac{\eta^2 r}{72}\right),
\end{equation}
where $B\sim\mathrm{Beta}\!\left(\frac r2,\frac{d-r}{2}\right)$.

We impose two conditions on $r$:
\begin{align}
\label{eq:r-lower}
2s\exp\!\left(-\frac{\eta^2 r}{72}\right)
&\le \frac12\,\delta_{\mathrm{Gauss}}(\varepsilon),
\\
\label{eq:r-upper}
\alpha=(1+\eta)\frac{r}{d}
&\le \alpha_0.
\end{align}
Condition \eqref{eq:r-upper} is equivalent to
$r \le \frac{\alpha_0}{1+\eta}\,d$, which is ensured whenever $r\ll d$.
Condition \eqref{eq:r-lower} holds whenever
\[
r \;\ge\; \frac{72}{\eta^2}\log\!\left(\frac{4s}{\delta_{\mathrm{Gauss}}(\varepsilon)}\right),
\]
which is of the form $r \gtrsim \log s$ up to constant factors.
Under these conditions we have, by monotonicity of $T(\varepsilon;\cdot)$ and \eqref{eq:alpha0-def},
\[
T(\varepsilon;\alpha\mu)\;\le\;T(\varepsilon;\alpha_0\mu)
\;=\;\frac12\,\delta_{\mathrm{Gauss}}(\varepsilon),
\]
and by \eqref{eq:tail-bound-use} and \eqref{eq:r-lower},
\[
s\left[1-I_\alpha\!\left(\frac r2,\frac{d-r}{2}\right)\right]
\;\le\;\frac12\,\delta_{\mathrm{Gauss}}(\varepsilon).
\]
Therefore,
\[
\delta_{\mathrm{ours}}(\varepsilon;\alpha)
\;\le\;
T(\varepsilon;\alpha\mu)
+
s\left[1-I_\alpha\!\left(\frac r2,\frac{d-r}{2}\right)\right]
\;\le\;
\delta_{\mathrm{Gauss}}(\varepsilon).
\]
Moreover, the inequality is strict whenever the two half-budget bounds above are strict (e.g., by taking
$r$ slightly larger than the lower threshold and slightly smaller than the upper threshold), which yields
\[
\delta_{\mathrm{ours}}(\varepsilon;\alpha)\;<\;\delta_{\mathrm{Gauss}}(\varepsilon).
\]

Finally, by construction $\alpha=(1+\eta)\frac{r}{d}$, so $\alpha\approx r/d$, concluding the proof.
\end{proof}

\section{LoRA}

\begin{algorithm}[h]
\caption{One LoRA step with frozen $A$ (single layer)}
\label{algo:LoRA-w-basic-single}
\textbf{Input:} pretrained weights $W_{0}\in\mathbb{R}^{n\times d}$; rank $r<\min\{n,d\}$;
dataset size $N$; loss $\mathcal{L}$; step size $\eta$; minibatch size $B_{\mathrm{mb}}$.
\begin{algorithmic}
  \STATE Sample minibatch $\mathcal{B}\subset [N]$ with $|\mathcal{B}|=B_{\mathrm{mb}}$
  (e.g.\ Poisson rate $q=B_{\mathrm{mb}}/N$).
  \STATE \textbf{Sample} $A \sim \mathcal{N}(0,1/r)^{r\times d}$ and freeze it.
  \STATE \textbf{Initialize} $B \gets 0 \in \mathbb{R}^{n\times r}$.
  \STATE Form effective weights for this step: $W \gets W_{0}+BA$.
  \STATE Compute gradient $G \gets \nabla_{B}\mathcal{L}\!\big(W,\,\mathcal{B}\big)\in\mathbb{R}^{n\times r}$.
  \STATE Update $B \gets B - \eta\, G$.
  \STATE Subsequent forward passes use $W = W_{0}+BA$.
\end{algorithmic}
\end{algorithm}

LoRA adapts a pretrained weight matrix $W_0\in\mathbb{R}^{n\times d}$ via a low-rank update
$W_{\mathrm{eff}} = W_0 + BA$, where $B\in\mathbb{R}^{n\times r}$ and $A\in\mathbb{R}^{r\times d}$ with $r\ll \min\{n,d\}$.
In LoRA-FA, $A$ is sampled once at initialization and then frozen, and only $B$ is trained. Starting from $B_0$ (typically $0$),
each step uses $W=W_0+B_{t-1}A$ on a minibatch $\mathcal{B}_t$ and updates
$B_t = B_{t-1} - \eta\,\nabla_B\mathcal{L}(W;\mathcal{B}_t)$.

To train LoRA-FA on private data, it suffices to make the procedure that outputs $B_t$ is DP, since $A$ and $W_0$ are fixed (and not data-dependent). To the best of our knowledge, no DP-LoRA algorithm has been proposed for the
fixed $A$ setting, however to privatize LoRA-FA in the same spirit as~\cite{sun2024improving} we apply \emph{per-example} gradient clipping
to the gradients with respect to $B$ and add Gaussian noise to the averaged clipped gradient.
Concretely, for a minibatch $\mathcal{B}_t$ we compute per-example gradients
$G_{t,i}=\nabla_B \ell(W;i)$, clip each to Frobenius norm at most $\beta$,
$\widetilde G_{t,i}=\min\!\left(1,\frac{\beta}{\|G_{t,i}\|_F}\right)G_{t,i}$,
and form the noisy DP gradient
\[
\widehat G_t \;=\; \frac{1}{B_{\mathrm{mb}}}\sum_{i\in\mathcal{B}_t}\widetilde G_{t,i}
\;+\;\frac{\sigma}{B_{\mathrm{mb}}}E_t,
\qquad (E_t)_{jk}\overset{\text{i.i.d.}}{\sim} \mathcal N(0,1)\ 
\]
which is then used to update $B$ via $B_t = B_{t-1}-\eta\,\widehat G_t$.
Because each $\widetilde G_{t,i}$ has Frobenius norm at most $\beta$, changing one example affects the
summed clipped gradient by at most $2\beta$ in Frobenius norm. Adding Gaussian noise calibrated to $\beta$
yields a differentially private update. Over $T$ iterations (and with minibatch subsampling), the overall
privacy guarantee for the final released $B_T$ (and hence $W_0+B_TA$) follows by standard DP-SGD
privacy accounting/composition.~\Cref{algo:dp-lora-FA} describes one step of one layer of DP-LoRA-FA.

\begin{algorithm}[h]
\caption{DP-LoRA-FA (one layer, $T$ steps)}
\label{algo:dp-lora-FA}
\textbf{Input:} pretrained weight matrix $W_0\in\mathbb{R}^{n\times d}$; rank $r\ll \min\{n,d\}$;
steps $T$; dataset size $N$; minibatch size $B_{\mathrm{mb}}$; per-example loss $\ell(\cdot;\cdot)$; step size $\eta$;
clipping norm $\beta$; privacy parameters $(\varepsilon,\delta)$.
\begin{algorithmic}
  \STATE \textbf{Set} privacy noise $\sigma \gets \frac{2\beta\sqrt{2\ln(1.25/\delta)}}{\varepsilon}$.
  \STATE \textbf{Sample} $A \sim \mathcal{N}(0,1/r)^{r\times d}$ and freeze it.
  \STATE \textbf{Initialize} $B \gets 0 \in \mathbb{R}^{n\times r}$.
  \FOR{$t=1,\dots,T$}
    \STATE Sample minibatch $\mathcal{B}_t\subset [N]$ with $|\mathcal{B}_t|=B_{\mathrm{mb}}$
    (e.g.\ Poisson rate $q=B_{\mathrm{mb}}/N$).
    \STATE \textbf{Effective weights:} $W \gets W_0 + BA$.
    \FOR{each $i\in\mathcal{B}_t$}
      \STATE $G_{t,i}\gets \nabla_B\ell(W;i)$,
      \quad $\widetilde G_{t,i}\gets \min\!\left(1,\frac{\beta}{\|G_{t,i}\|_F}\right)G_{t,i}$.
    \ENDFOR
    \STATE \textbf{DP gradient:}
    $\widehat G_t \gets \frac{1}{B_{\mathrm{mb}}}\sum_{i\in\mathcal{B}_t}\widetilde G_{t,i}
    \;+\;\frac{\sigma}{B_{\mathrm{mb}}}E_t$,
    where $(E_t)_{jk}\overset{\text{i.i.d.}}{\sim} \mathcal N(0,1)$.
    \STATE \textbf{Update:} $B \gets B - \eta\,\widehat G_t$.
  \ENDFOR
  \STATE \textbf{Output:} $B$ (and implicitly $W=W_0+BA$).
\end{algorithmic}
\end{algorithm}


For $W=W_0+BA$, the gradient with respect to $B$ follows by the chain rule:
\[
\nabla_B \mathcal{L}(W;\mathcal{B}_t)
\;=\;
\nabla_{W}\mathcal{L}(W;\mathcal{B}_t)\,A^\top,
\]


Therefore, we can re-write the training dynamic above as, set $W = W_0 + B_0 A_0$ (where we assume $B_0$ was initialized as a 0 matrix) and update
\[B_1 = B_0 - \eta\nabla_{W}\mathcal{L}(W;\mathcal{B}_t)\,A^\top\]
further 
\[
    W = W_0 + B_1A = W_0 + B_0A - \eta \nabla_{W}\mathcal{L}(W;\mathcal{B}_t)\,A^\top A = W_0 - - \eta \nabla_{W}\mathcal{L}(W;\mathcal{B}_t)\,A^\top A
\]
so if we wanted to privatize one single step of one-layer of LoRA-FA we would exclusively need to privatize 
\[
    \nabla_{W}\mathcal{L}(W;\mathcal{B}_t)\,A^\top A
\]
as $W_0, B_0$ are fixed. We can achieve this by clipping and adding noise: 
\[
    \left(\min \left(1, \frac{\beta'}{\| \nabla_{W}\mathcal{L}(W;\mathcal{B}_t)\|_F} \right)  \nabla_{W}\mathcal{L}(W;\mathcal{B}_t) + \sigma' E' \right) A^\top A
\]
Note that this is exactly the noisy projection mechanism studied in~\Cref{subsec:matproj-small-r}. We want to remark that $W$ in the above equation does not depend on $A$ as $B_0 = 0$, so we have that $A$ is independent of the gradient and the privacy results of~\Cref{subsec:matproj-small-r} are applicable. For mutlistep and multilayer implementation details please se~\Cref{app:experiment-details}.

\begin{algorithm}[h]
\caption{Noisy Projection Mechanism}
\label{algo:noisy-proj-mech}
\textbf{Input:} pretrained $W_0\in\mathbb{R}^{n\times d}$; step size $\eta$; dataset size $N$; minibatch size $B_{\mathrm{mb}}$ (or full-batch); loss $\mathcal L$;
clipping level $\beta'$; privacy parameters $(\varepsilon,\delta)$.\\
\textbf{Notation:} $\mathrm{clip}_{\beta'}(X)=\min\!\left(1,\frac{\beta'}{\|X\|_F}\right)X$.
\begin{algorithmic}
    \STATE \textbf{Set} $B \gets 0 \in \mathbb{R}^{n \times r}$
    \STATE \textbf{Sample} $A \sim \cN(0, 1/r)^{r \times d}$
    \STATE \textbf{Set} $W_{\text{eff}} \gets W_0 + BA$ 
    \STATE Sample minibatch $\mathcal{B}_t\subset[N]$ with $|\mathcal{B}_t|=B_{\mathrm{mb}}$ (or take $\mathcal{B}_t=[N]$).
    \STATE Compute gradient $G \gets \nabla_W \mathcal L(W_{\text{eff}};\mathcal{B}_t)\in\mathbb{R}^{n\times d}$.
  \STATE \textbf{DP gradient:}
  $\widehat G \gets (\mathrm{clip}_{\beta'}(G) + \sigma' E')A^\top A$,
  where $E'_{ij}\sim\mathcal N(0,1)$ i.i.d.
\end{algorithmic}
\end{algorithm}

\subsection{Comparison between one step of DP-LoRA-FA and Projection Mechanism}\label{app:subsec-comparison-dp-lora-small-r}
In order to compare~\Cref{algo:noisy-proj-mech} to~\Cref{algo:dp-lora-FA} we note that the final weights of DP-LoRA-FA (if we neglect clipping) are 

\begin{align*}
    W_T &= W_0 + B A \\
    &= W_0 - \eta \left( \sum_{t=1}^{T} \nabla_B\mathcal{L}  + \sigma E \right) A
\end{align*}

for our projection mechanism we have 

\begin{align*}
    W_T   &= W_0 - \eta \left( \sum_{t=1}^T \nabla_W \mathcal{L} + \sigma' E' \right) A_t^\top A_t     \\
    &=  W_0 - \eta \left( \sum_{t=1}^T \nabla_W \mathcal{L} A_t^\top + \sigma' E' A_t^\top \right) A_t \\
    & = W_0 - \eta \left( \sum_{t=1}^T \nabla_B\mathcal{L} + \sigma' E' A_t^\top \right) A_t 
\end{align*}

This means to compare of one step in one layer if we assume we set the clipping threshold in a way that we won't clip with high probability (we choose $\beta$ and $\beta'$ so that we do not clip most of the time) then we can simply compare the privacy of the releases 
\[
    \nabla_B\mathcal{L} + \sigma \cdot E
\]
to 
\[
    \nabla_B\mathcal{L}  + \sigma' E'A^\top.
\]

Where we recall that $\sigma = \frac{2\beta\sqrt{2\ln(1.25/\delta)}}{\varepsilon}$ and $\sigma' = \frac{2\beta'\sqrt{2\ln(1.25/\delta)}}{\varepsilon}$. Further our analysis in~\Cref{cor:small-r-delta-improvement} shows that the random projection $AA^\top$ will contract the sensitivity $\Delta V$ with a multiplier of $\alpha \approx \frac{r}{d}$ w.h.p. for $r$ small. This means the projection mechanism we implement to compare to DP-LoRA-FA actually has a noise multiplier of $\sqrt{r/d}\sigma'$. This means we are interested in comparing the following two mechanisms:
\begin{align*}
    \cM_{\text{DP-LoRA}}\left(\nabla_B\mathcal{L}\right) &=   \nabla_B\mathcal{L} + \sigma E \\
    \cM_{\text{ProjMech}} \left(\nabla_B\mathcal{L}\right) &= \nabla_B\mathcal{L} + \sigma' \sqrt{\frac{r}{d}}E' A^\top
\end{align*}

We would like to show that $\cM_{\text{ProjMech}}$ $(\varepsilon, \delta)$-DP $\implies$ $\cM_{\text{DP-LoRA}}$ $(\tilde{\varepsilon}, \tilde{\delta})$-DP. For this we need to understand how the variances of $\sigma E$ and $\sigma' \sqrt{\frac{r}{d}}E' A^\top$ relate to each other. So first of all we will need to determine how $\sigma$ and $\sigma'$ differ. Notice they only differ in their clipping thresholds $\beta$ and $\beta'$. We chose $\beta$ in order to control the norm of $\nabla_B\mathcal{L}$ we choose $\beta'$ in order to control the norm of $\nabla_W \mathcal{L}$. We note that $\nabla_W \mathcal{L} A^\top = \nabla_B\mathcal{L}$ by the chain rule, where $A \in \mathbb{R}^{r \times d}$ with coordinates sampled i.i.d. from $\cN(0,1/r)$. So we will use the Johnson-Lindenstrauss Lemma to compare the norms of $\nabla_W \mathcal{L} A^\top $ and $\nabla_W \mathcal{L}$ to find comparable choices of $\beta$ and $\beta'$.

\begin{lem}[Johnson-Lindenstrauss \citep{dasgupta2003elementary}]
    For $A$ a random matrix $A \in \mathbb{R}^{r \times d}$ obtained from sampling the coordinates i.i.d from $\cN(0,1/r)$ and $x \in \mathbb{R}^d$ any non zero vector we have that 
    \[
      (1 - \zeta)  \|x \|_2^2 \leq  \| A x \|_2^2 \leq (1 + \zeta)  \|x \|_2^2
    \]
    with probability $1 - 2e^{-\frac{r}{2}(\frac{1}{2}\zeta^2 - \frac{1}{3}\zeta^3)}$. 

    By the union bound, the probability that this relation is true for $x_1, \dots, x_n$ is greater than $1 - 2ne^{-\frac{r}{2}(\frac{1}{2}\zeta^2 - \frac{1}{3}\zeta^3)}$
    \end{lem}

A common simplification is that for $\zeta\in(0,1)$,
\[
\frac12\,\zeta^2-\frac13\,\zeta^3
\;\ge\;
\frac16\,\zeta^2,
\]
since
\[
\frac12\,\zeta^2-\frac13\,\zeta^3-\frac16\,\zeta^2
=
\frac13\,\zeta^2(1-\zeta)\;\ge\;0.
\]
Note $\| \nabla_W \mathcal{L} A^\top \|_F = \| A\nabla_W \mathcal{L}^\top\|_F$. So if we define $V = \nabla_W\mathcal{L}^\top$ we have that $V \in \mathbb{R}^{d \times n}$ and we can for ease of notation equivalently analyse $\| A V\|_F$ compared to $\| V \|_F$. For $V \in \mathbb{R}^{d \times n}$ it suffices to require
\[
2n\exp\!\left(-\frac{r}{12}\zeta^2\right)\le \delta_{\text{JL}},
\qquad\text{which implies}\qquad
\zeta \;\ge\; \sqrt{\frac{12\log(2n/\delta_{\text{JL}})}{r}}.
\]
Hence, setting $\zeta = \sqrt{\frac{12\log(2n/\delta_{\text{JL}})}{r}}$, with probability at least $1-\delta_{\text{JL}}$,
\[
\frac{\|AV\|_{F}}{\sqrt{1+\zeta}}
\;\le\;
\|V\|_{F}
\;\le\;
\frac{\|AV\|_{F}}{\sqrt{1-\zeta}}.
\]
as $\|V\|_F^2 = \sum_j \| v_j\|_2^2$, for $v_j$ the columns of $V$. 

Recalling that $\beta$ is the clipping bound for $\| A V \|_F = \| \nabla_W \mathcal{L} A ^\top\|_F = \|\nabla_B \mathcal{L} \|_F$ and $\beta'$ for $\|V\|_F = \|\nabla_W \mathcal{L} \|_F $ we have that whenever 
\begin{align*}
    &\|\nabla_W \mathcal{L} \|_F \leq \beta' \implies \|\nabla_B \mathcal{L} \|_F \leq \sqrt{1 + \zeta}\beta' \\
    &\| \nabla_B \mathcal{L} \|_F \leq \beta \implies \|\nabla_W \mathcal{L} \|_F \leq \frac{1}{\sqrt{1 - \zeta}} \beta
\end{align*}
For simplicity we set $\beta = \beta'$. This leads to $\sigma = \sigma'$. Next we want to compare the distribution of $\cM_{\text{DP-LoRA}}$ and $\cM_{\text{ProjMech}}$. Note they share the same mean ($\frac{\partial \mathcal{L}}{\partial B}$) so we are left to investigate their variance. For $\cM_{\text{ProjMech}}$ the variance comes from $\sigma \sqrt{\frac{r}{d}}E' A^\top$ (recall we set $\beta = \beta'$, so $\sigma = \sigma'$).

Let
\[
C \;\coloneqq\; \frac{r}{d}\,A A^\top \in \mathbb{R}^{r\times r}.
\]
Since each row of \(E'\in\mathbb{R}^{n\times d}\) is \(\mathcal{N}(0,I_d)\), conditional on \(A\) we have
\[
\sqrt{\frac{r}{d}}\,E'A^\top \,\big|\, A \;\sim\; \mathcal{MN}\!\left(0,\ I_n,\ C\right),
\]
equivalently (by row-stacking),
\[
\operatorname{vec}\!\left(\sqrt{\frac{r}{d}}\,E'A^\top\right)\Big|A
\;\sim\;
\mathcal{N}\!\left(0,\ I_n\otimes C\right).
\]
Hence the mechanism
\[
\mathcal{M}_{\text{ProjMech}}(D)\;=\;G(D)+\sigma\sqrt{\frac{r}{d}}\,E'A^\top
\]
is (conditionally on \(A\)) a Gaussian mechanism with covariance
\[
\Sigma_{\text{ProjMech}}(A)\;=\;\sigma^2\,(I_n\otimes C).
\]
In comparison, the isotropic mechanism
\[
\mathcal{M}_{\text{DP-LoRA}}(D)\;=\;G(D)+\sigma E,
\qquad E\sim \mathcal{MN}(0,I_n,I_r),
\]
has covariance
\[
\Sigma_{\text{DP-LoRA}}\;=\;\sigma^2\,(I_n\otimes I_r).
\]

To compare \(C\) to \(I_r\) note \(A_{ij}\sim\mathcal{N}(0,1/r)\) and write
\(A=\frac{1}{\sqrt r}Z^\top\) with \(Z\in\mathbb{R}^{d\times r}\) i.i.d.\ \(\mathcal{N}(0,1)\), so that
\[
C=\frac{r}{d}AA^\top \;=\; \frac{1}{d}Z^\top Z.
\]
were we recall $r < d$. We can use \Cref{lem:wishart-ZZt} to bound the eigenvalues of C and hence the amount of noise we add. Note that \Cref{lem:wishart-ZZt} is stated for \(W=\frac{1}{r}ZZ^\top\), whose nonzero eigenvalues are of order \(d/r\). In our setting we need bounds for
\(
C=\frac{1}{d}Z^\top Z,
\)
which has the same nonzero spectrum as \(\frac{1}{d}ZZ^\top\). Since \(\frac{1}{d}ZZ^\top=\frac{r}{d}\,W\), the eigenvalues rescale by the factor \(r/d\), turning the \(\sqrt{d/r}\) scale in \Cref{lem:wishart-ZZt} into a \(\sqrt{r/d}\) deviation around \(1\). Consequently, for every \(t\ge 0\), with probability at least \(1-2e^{-t^2/2}\),
\[
\left(1-\sqrt{\frac{r}{d}}-\frac{t}{\sqrt d}\right)^2 I_r
\;\preceq\;
C
\;\preceq\;
\left(1+\sqrt{\frac{r}{d}}+\frac{t}{\sqrt d}\right)^2 I_r.
\]

Write \(\varepsilon_{\mathrm{proj}}(\alpha)\) and \(\varepsilon_{\mathrm{iso}}(\alpha)\) for the Rényi DP parameters of $\cM_{\text{ProjMech}}$ and $\cM_{\text{DP-LoRA}}$, respectively (of order \(\alpha>1\)) at the same noise scale \(\sigma\) and the same clipping/sensitivity bound.

On the good event \(\mathcal{E}\) where the spectrum of \(C=\tfrac{r}{d}AA^\top\) concentrates, we have for all \(\alpha>1\),
\[
(1-\eta)^2\,\varepsilon_{\mathrm{proj}}(\alpha)
\;\le\;
\varepsilon_{\mathrm{iso}}(\alpha)
\;\le\;
(1+\eta)^2\,\varepsilon_{\mathrm{proj}}(\alpha),
\qquad
\eta \;\coloneqq\; \sqrt{\frac{r}{d}}+\sqrt{\frac{2\log(2/\gamma)}{d}}.
\]
Consequently, after converting from RDP to \((\varepsilon,\delta)\)-DP, the isotropic mechanism satisfies
\[
\mathcal{M}_{\mathrm{ProjMech}}\ \ (\varepsilon_{\mathrm{proj}},\delta)\text{-DP}
\quad\Longrightarrow\quad
\mathcal{M}_{\mathrm{DP-LoRA}}\ \ (\varepsilon_{\mathrm{iso}},\delta+\gamma)\text{-DP},
\]
with \(\varepsilon_{\mathrm{iso}}\) within a multiplicative factor \(1\pm O(\sqrt{r/d})\) of \(\varepsilon_{\mathrm{proj}}\) on \(\mathcal{E}\)
(assuming \(d\gg r\) and \(d\gg\log(1/\gamma)\)).

\subsection{DP Projection Mechanism without clipping}

If we are in a setting where we know we do not have to clip, as the gradients are naturally bounded then we can rewrite the projection mechanism for LoRA so that only the gradients with respect to $B$ are being used. Allowing us to regain computational and memory efficiency, which is one of the main motivations of LoRA.

\begin{algorithm}[H]
    \caption{Projection Mechanism without clipping}\label{algo:lora-add-noise-1}
    \textbf{Input:} Pre-trained model parameters $W_0 \in \mathbb{R}^{n \times d}$, low dimension $r$ with $r < d,n$, training data $x \in \mathbb{R}^{d}$, loss function $\mathcal{L}$, number of rounds $T$, step size $\eta$, privacy noise $\sigma>0$
    \begin{algorithmic}
        \STATE Initialize $B_0 \in \mathbb{R}^{n \times r}$ and $A \in \mathbb{R}^{r \times d}$ randomly
        \FOR{$t = 1, \dots T-1$}
            \STATE $W_t \gets W_{0} + B_{t-1} A$ \COMMENT{update model}
            \STATE $y_t \gets W_t x$ \COMMENT{evaluate}
            \STATE $G_t \gets  \frac{\partial \mathcal{L}(y)}{\partial B}|_{y = y_t}$ \COMMENT{calculate gradient}
            \STATE $B_t = B_{t-1} + \eta g_t$ \COMMENT{update B}
        \ENDFOR
        \STATE $E \sim \mathcal{MN}(0, I_d, I_k)$ \COMMENT{sample noise}
        \STATE $\tilde{G} \gets \sum_{t=1}^{T} G_t  + \sigma E A^\top$ \COMMENT{privatize gradients}
        \STATE $\tilde{W} \gets W_0 - \eta \tilde{G} A$
    \end{algorithmic}
\end{algorithm}

\section{Experiment details}\label{app:experiment-details}

In this section we discuss the experimental details. For our experiments on membership inference attack, we run on a single GPU, and for our experiments on comparing performance of M2, DP-LORA and DP-SGD, we run on CPU. 

\subsection{Membership inference attack}\label{app:exp-mia}

\paragraph{Datasets and pre-processing.}
Our target task is CIFAR-10. In the \emph{pretrain+fine-tune} pipeline, we pretrain on CIFAR-100 and then fine-tune on CIFAR-10.
All CIFAR inputs are normalized using per-channel mean $(0.4914, 0.4822, 0.4465)$ and standard deviation $(0.2023, 0.1994, 0.2010)$.
During training we apply random crop ($32\times 32$ with padding 4) and random horizontal flip; during evaluation we apply normalization only.

\paragraph{Attack setup and notation.}
Let $D$ be a CIFAR-10 subset of size $|D|=5000$. The adversary selects a \emph{canary} $(x_q, y_q)$ and aims to infer whether it was included in training. Our membership inference evaluation procudure mimics the following membership inference game. 
We denote the (randomized) target-training procedure by $\cA$ and the resulting trained model by $f$.
For each attack trial, the model trainer samples
\[
b \sim \mathrm{Bernoulli}(0.5),
\]
and trains the target model as
\[
f_{\text{target}} \gets
\begin{cases}
\cA\!\left(D \cup \{(x_q,y_q)\}\right) & \text{if } b=1 \quad(\text{IN}),\\
\cA(D) & \text{if } b=0 \quad(\text{OUT}).
\end{cases}
\]
Given $f_{\text{target}}$, $D$, $(x_q,y_q)$, and knowledge of $\cA$, the adversary outputs a guess $\hat b \in \{0,1\}$.

Next, we instantiate our training algorithm, canary crafting algorithm, and membership inference evaluation method. 
\paragraph{Training algorithm $\cA$ (model, optimizer, and schedule).}
We use a CNN with three $3\times 3$ convolution layers with channel sizes $32/64/128$, each followed by ReLU and $2\times 2$ max-pooling, then two fully-connected layers.
All training uses LoRA-FA gradient update based on SGD with momentum $0.9$ and cross-entropy loss.
Learning-rate schedules are either (i) cosine annealing over all iterations with minimum learning rate $\eta_{\min}=10^{-4}$, or (ii) a step schedule.

\paragraph{LoRA-FA gradient updates.}
For each trainable weight tensor $W$ in each layer, let $G$ denote its minibatch gradient.
We reshape $G$ into a matrix $\tilde{G}\in\mathbb{R}^{d_o\times d_i}$ by flattening all non-output dimensions.
At initialization, we sample a \emph{fixed} Gaussian matrix $A\in\mathbb{R}^{r\times d_i}$ with i.i.d.\ entries $\mathcal{N}(0,1)$ (one such matrix per layer), and define
\[
M \;=\; \frac{1}{r}A^\top A \in \mathbb{R}^{d_i\times d_i}.
\]
We then apply a rank-$r$ right-projection to the gradient,
\begin{equation}
\tilde{G} \;\gets\; \tilde{G}M,
\label{eq:grad-compress}
\end{equation}
and reshape the projected gradient back to the original tensor shape before performing the optimizer update.

\paragraph{Canary construction.}
The adversary samples the canary input as $x_q \sim \cN(0,1)$ with shape $3\times 32\times 32$.
It then trains a reference model $f = \cA(D)$ using the same procedure $\cA$ as the model trainer, and sets the canary label to be the least-likely class under $f$:
\[
y_q \;=\; \arg\min_{y\in\{0,\ldots,9\}} [f(x_q)]_y,
\]
where $f(x_q)$ denotes the logits and $[f(x_q)]_y$ is the logit for class $y$.
\paragraph{Membership inference evaluation protocol.}
To quantify membership leakage for the canary $(x_q,y_q)$ under training algorithm $\cA$, we train two collections of shadow models:
\begin{itemize}
    \item IN models: $N_{\text{in}}=1000$ models $\{f_i^{\text{in}}\}_{i=1}^{N_{\text{in}}}$ trained as $f_i^{\text{in}} \gets \cA\!\left(D \cup \{(x_q,y_q)\}\right)$.
    \item OUT models $N_{\text{out}}=1000$ models $\{f_j^{\text{out}}\}_{j=1}^{N_{\text{out}}}$ trained as $f_j^{\text{out}} \gets \cA(D)$.
\end{itemize}
For each trained model we compute the canary loss,
\[
s_i^{\text{in}} = \ell\!\left(f_i^{\text{in}}, (x_q,y_q)\right),
\qquad
s_j^{\text{out}} = \ell\!\left(f_j^{\text{out}}, (x_q,y_q)\right).
\]
We treat $\ell(\cdot)$ as a membership score: lower loss indicates higher likelihood of membership.
We then estimate ROC-AUC and the best balanced accuracy directly from the two empirical score sets
$\{s_i^{\text{in}}\}_{i=1}^{N_{\text{in}}}$ and $\{s_j^{\text{out}}\}_{j=1}^{N_{\text{out}}}$ by sweeping a threshold over all unique loss values.


\paragraph{Metrics.}
Given vectors of true memberships $\bm{b}$ and adversary predictions/scores $\bm{\hat b}$ (or scalar scores such as losses), we report:
(i) ROC-AUC, and
(ii) the best balanced accuracy obtained by sweeping thresholds over all unique scores, where balanced accuracy is
$\tfrac{1}{2}(\mathrm{TPR}+\mathrm{TNR})$.

\paragraph{Additional results on Noisy projection mechanism} We then run noisy LoRA via gradient descent with $M(V+E)$, matching the small-$r$ regime in \Cref{subsec:matproj-small-r}. Results based on training 200 IN\/OUT models with~\Cref{eq:M2} are summarized in \Cref{tab:perf_vs_r_noise}. As $r$ increases, MIA success increases, consistent with \Cref{thm:matrix-dp-proj-mechanism-small-r}. Additionally, \Cref{tab:mia-large-r} summarizes the results for MIA on~\Cref{eq:M1} for larger $r$, 

\begin{table*}[t]
\centering
\small
\setlength{\tabcolsep}{6pt}
\renewcommand{\arraystretch}{1.15}

\begin{subtable}[t]{\textwidth}
\centering
\caption{\textbf{AUC}}
\begin{tabular}{lcccccc}
\toprule
\textbf{Noise} & \textbf{16} & \textbf{64} & \textbf{128} & \textbf{256} & \textbf{384} & \textbf{512} \\
\midrule
0.1 & 0.78 & 0.91 & 0.96 & 0.99 & 1.00 & 1.00 \\
0.5 & 0.56 & 0.70 & 0.74 & 0.76 & 0.82 & 0.86 \\
\bottomrule
\end{tabular}
\end{subtable}

\vspace{0.6em}

\begin{subtable}[t]{\textwidth}
\centering
\caption{\textbf{Balanced accuracy}}
\begin{tabular}{lcccccc}
\toprule
\textbf{Noise} & \textbf{16} & \textbf{64} & \textbf{128} & \textbf{256} & \textbf{384} & \textbf{512} \\
\midrule
0.1 & 0.74 & 0.83 & 0.92 & 0.98 & 0.98 & 0.98 \\
0.5 & 0.56 & 0.67 & 0.70 & 0.71 & 0.75 & 0.80 \\
\bottomrule
\end{tabular}
\end{subtable}

\vspace{0.6em}

\begin{subtable}[t]{\textwidth}
\centering
\caption{\textbf{Test accuracy (\%)}}
\begin{tabular}{lcccccc}
\toprule
\textbf{Noise} & \textbf{16} & \textbf{64} & \textbf{128} & \textbf{256} & \textbf{384} & \textbf{512} \\
\midrule
0.1 & 44.61 & 49.98 & 54.52 & 57.94 & 60.38 & 61.70 \\
0.5 & 33.13 & 37.14 & 42.57 & 46.48 & 48.26 & 50.92 \\
\bottomrule
\end{tabular}
\end{subtable}

\caption{Performance vs.\ projection rank $r$ for two noise levels.}
\label{tab:perf_vs_r_noise}
\end{table*}

\begin{table*}[t]
\centering
\small
\setlength{\tabcolsep}{6pt}
\renewcommand{\arraystretch}{1.15}
\begin{tabular}{lcccccc}
\toprule
 & \multicolumn{6}{c}{\textbf{Rank $r$}} \\
\cmidrule(lr){2-7}
\textbf{Metric} & \textbf{32} & \textbf{63} & \textbf{128} & \textbf{512} & \textbf{800} & \textbf{1000} \\
\midrule
Avg.\ test acc.\ (\%) & 35.94 & 39.70 & 43.54 & 47.21 & 50.26 & 51.58 \\
Balanced acc.\ (\%)   & 61.50 & 65.00 & 67.50 & 72.00 & 72.00 & 81.00 \\
AUC (\%)              & 61.92 & 68.77 & 71.54 & 77.67 & 78.61 & 86.91 \\
\bottomrule
\end{tabular}
\caption{MIA performance (200 trials) vs.\ projection rank $r$ at noise multiplier $0.5$ ($\varepsilon\approx 136.05$, $\delta=10^{-5}$).}
\label{tab:mia-large-r}
\end{table*}

\subsection{Experimental details for the comparison between \eqref{eq:M2}, DP-LoRA, and DP-SGD}
\label{app:subsec-comparison-exp-details}

\paragraph{Setting.}
We follow the representation-learning setup of~\citet{pinto24pillar} and use their fixed feature extractor: a ResNet-50 pretrained with self-supervised learning on ImageNet-1K.
For each CIFAR-10 example $x_i$, we compute a $2048$-dimensional representation $z_i\in\mathbb{R}^{2048}$ (e.g., the pooled penultimate-layer feature), yielding a representation dataset
$\{(z_i,y_i)\}_{i=1}^n$ with $z_i\in\mathbb{R}^{2048}$ and $y_i\in\{0,\dots,9\}$.
We then train a \emph{linear} classifier on top of these frozen representations using three private training methods:
(i) DP-SGD~\citep{Abadi2016},
(ii) DP-LoRA-FA~(\Cref{algo:dp-lora-FA}), and
(iii) our noisy-projection mechanism~\eqref{eq:M2}~(\Cref{algo:noisy-proj-mech}) with the privacy accounting from \Cref{thm:matrix-dp-proj-mechanism-small-r}.
For method~(iii), we resample an independent random projection matrix (equivalently, $A_t$ and thus $M_t$) at every optimization step $t$.

\paragraph{Privacy accounting.}
\Cref{thm:matrix-dp-proj-mechanism-small-r} implies that a \emph{single} (clipped) gradient update is
$(\varepsilon,\,\delta_g+\delta_p)$-DP, where
\[
\delta_g = T\left(\varepsilon,\frac{\alpha\|\Delta V\|_F^2}{\sigma_G^2}\right)
\qquad\text{and}\qquad
\delta_p = s \br{1 - I_\alpha\br{\frac{r}{2}, \frac{d-r}{2}}}.
\]
Here, $\delta_g$ accounts for the privacy loss due to Gaussian noise addition (with noise level determined by $\sigma_G$),
and $\delta_p$ upper-bounds the probability that the random projection fails to satisfy the required ``good'' event.

In our experiments we run compressed stochastic gradient descent for $T$ steps, each step with a freshly sampled matrix $M$ as in~\Cref{eq:M2}, and thus must compose privacy across steps.
Although the mapping in~\eqref{eq:M2} is not unconditionally a standard Gaussian mechanism, it becomes one after conditioning on
\[
\mathcal{G}(M_t) \;:=\; \left\{ \frac{\|P_{M_t}\Delta V\|_F^2}{\|\Delta V\|_F^2} > \alpha \right\}.
\]
Conditioned on $\mathcal{G}(M_t)$ (for fixed $M_t$), the update is equivalent to a Gaussian mechanism with effective $\ell_2$-sensitivity
$\sqrt{\alpha}\,\|\Delta V\|_F$ and Gaussian noise determined by $\Xi$.
Therefore, conditioning on the intersection of ``good'' events across all steps,
$\bigcap_{t=1}^T \mathcal{G}(M_t)$, we can apply a standard tight Gaussian-mechanism accountant (e.g., via R\'enyi DP~\citep{IlyaRDP2017}) to compose the Gaussian part over $T$ steps (and incorporate privacy amplification by subsampling in the usual way).
By a union bound, the intersection event holds with probability at least $1 - T\delta_p$, contributing an additive $T\delta_p$ term to the overall $\delta$. Consequently, for fixed $(\delta_p,r,d)$ we compute the largest admissible $\alpha$, which yields an adapted effective sensitivity
$\sqrt{\alpha}\,\|\Delta V\|_F$ for the Gaussian accountant. We then choose the noise multiplier to match the target privacy budget accordingly.

\paragraph{Hyperparameter selection.}
For all three methods we use batch size $1024$.
For DP-SGD, we use the learning rate recommended by~\citet{pinto24pillar} and tune the remaining hyperparameters by grid search over:
number of epochs in $\{35,40,45,50\}$ and clipping threshold in $\{0.5,0.7,1,1.5,2\}$.
For DP-LoRA-FA, we tune over the same epoch and clipping grids, and additionally tune the learning rate in
$\{0.1,0.3,0.5,0.7,1\}$ and the LoRA rank $r$ in $\{32,64,128,256,512,700\}$.
For our mechanism, we set $\delta_p = 0.1\,\delta$ and optimize over the same hyperparameter grids as DP-LoRA-FA.

\begin{figure}[t]
    \centering
    \includegraphics[width=0.4\linewidth]{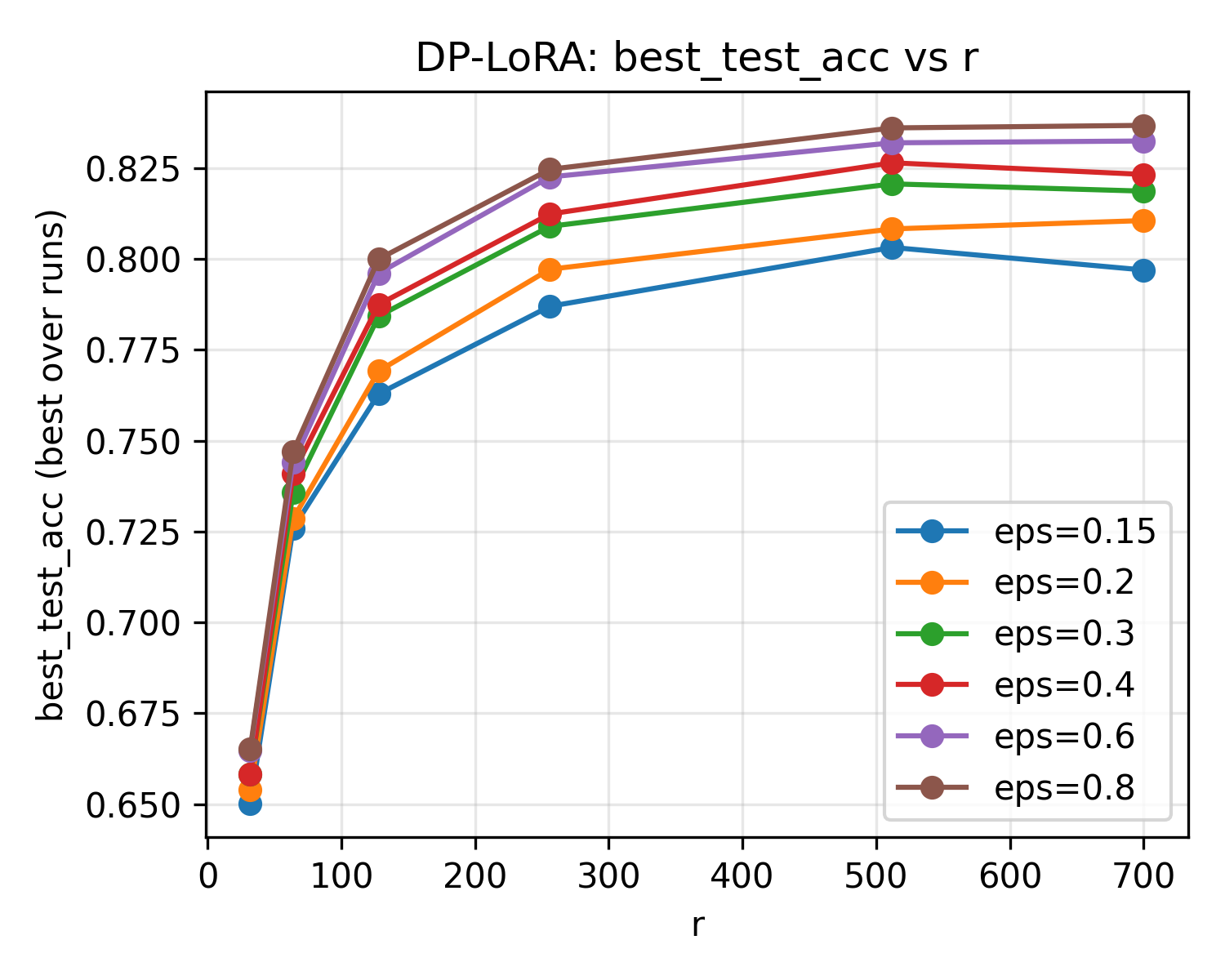}
    \includegraphics[width=0.4\linewidth]{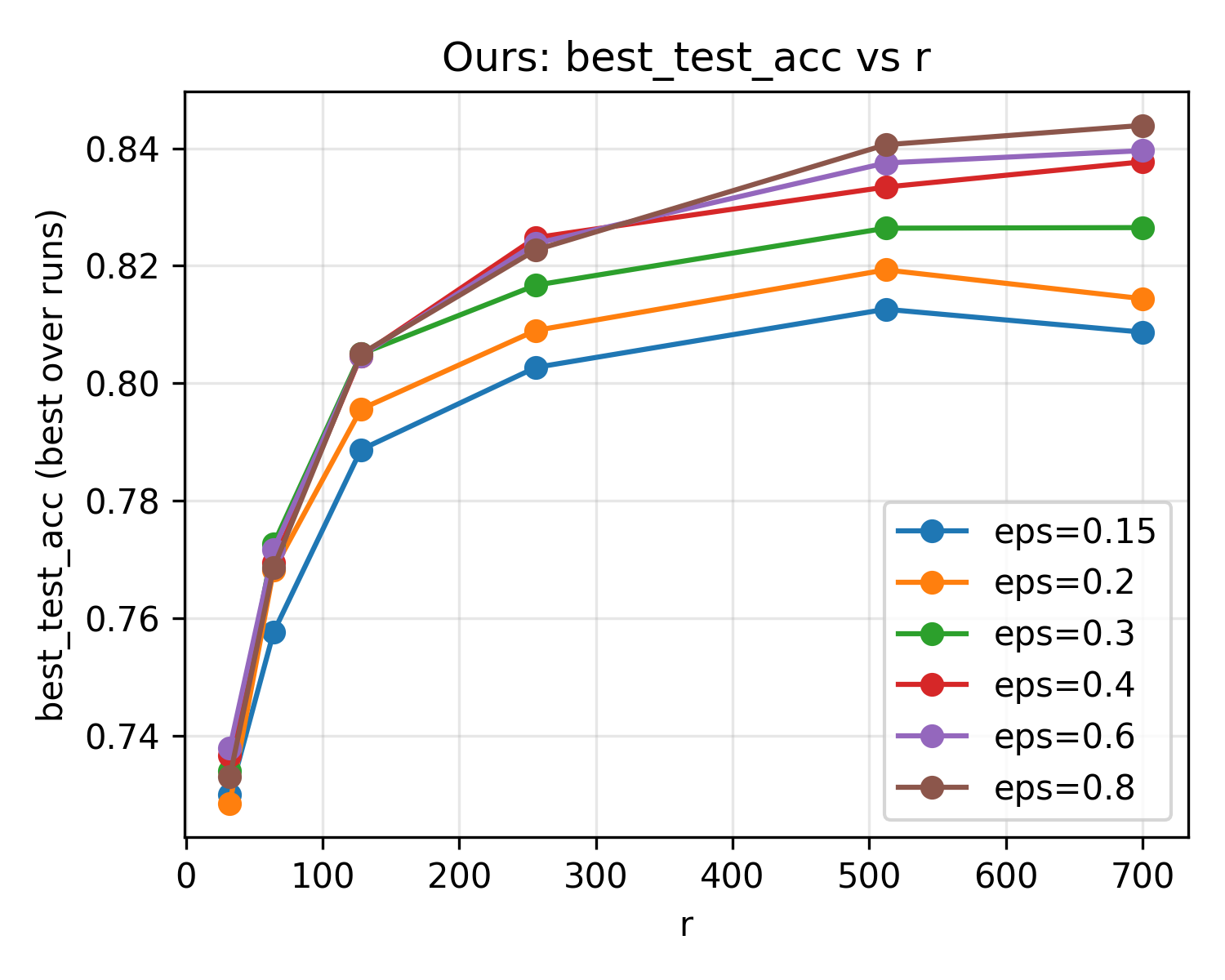}
    \caption{Best (non-private) test accuracy as a function of the rank $r$ for DP-LoRA-FA (left) and our noisy-projection mechanism (right), under the same target privacy budget used in the main comparison.}
    \label{fig:best-acc-vs-r-more-eps}
\end{figure}

\end{document}